\documentclass{article}

\usepackage[preprint]{neurips_2025}

\usepackage[utf8]{inputenc} 
\usepackage[T1]{fontenc}    
\usepackage{hyperref}       
\usepackage{url}            
\usepackage{booktabs}       
\usepackage{amsfonts}       
\usepackage{nicefrac}       
\usepackage{microtype}      
\usepackage{xcolor}         

\usepackage[utf8]{inputenc} 
\usepackage[T1]{fontenc}    
\usepackage{hyperref}       
\usepackage{url}            
\usepackage{booktabs}       
\usepackage{amsfonts}       
\usepackage{nicefrac}       
\usepackage{microtype}      
\usepackage{lipsum}
\usepackage{fancyhdr}       
\usepackage{graphicx}       

\definecolor{cornflowerblue}{rgb}{0.39, 0.58, 0.93}
\hypersetup{
    colorlinks=true,
    linkcolor=cornflowerblue,
    filecolor=magenta,      
    urlcolor=teal,
    citecolor=cornflowerblue,
    pdftitle={The Computational Intractability of Filtering for AI Alignment},
    pdfpagemode=FullScreen,
}

\usepackage{tikz}
\usetikzlibrary{arrows.meta, positioning}
\usepackage{adjustbox}

\usepackage{amssymb}
\usepackage{dirtytalk}

\usepackage{amsmath} 
\usepackage{mathtools}
\usepackage{amsthm}
\usepackage{cleveref}

\usepackage{algorithm}
\usepackage{algpseudocode}
\usepackage{comment}

\newtheorem{definition}{Definition}
\newtheorem{fact}{Fact}
\newtheorem{remark}{Remark}
\theoremstyle{plain} 
\newtheorem{theorem}{Theorem}
\newtheorem{lemma}{Lemma}

\newtheorem{note}{Note}

\graphicspath{{media/}}     

\newcommand{\N}{\mathbb{N}}
\newcommand{\Z}{\mathbb{Z}}

\newcommand{\e}{\epsilon}
\newcommand{\eps}{\epsilon}

\newcommand{\dist}{\mathcal{D}}

\newcommand{\Prob}{\mathbb{P}}

\newcommand{\secp}{\lambda}

\newcommand{\poly}{\mathsf{poly}}
\newcommand{\floor}[1]{\left\lfloor{#1}\right\rfloor}
\newcommand{\ceil}[1]{\lceil{#1}\rceil}

\DeclareMathOperator*{\argmin}{arg\,min}

\title{
On the Impossibility of Separating Intelligence from Judgment: The Computational Intractability of Filtering for AI Alignment
}

\author{
  Sarah Ball\thanks{Authors listed in alphabetical order.} \\
  Ludwig-Maximilians-Universität in Munich \\
  Munich Center for Machine Learning (MCML) \\
  \texttt{sarah.ball@stat.uni-muenchen.de}
  \And
    Greg Gluch \\
    University of California at Berkeley \\
    \texttt{gluch@berkeley.edu}
    \And
    Shafi Goldwasser \\
    University of California at Berkeley \\
  \And
  Frauke Kreuter \\
  Ludwig-Maximilians-Universität in Munich\\
  Munich Center for Machine Learning (MCML) \\
  JPSM University of Maryland \\
   \And
   Omer Reingold \\
   Stanford University \\
   \And
   Guy N. Rothblum \\
   Apple \\
}

\begin{document}

\maketitle

\begin{abstract}
With the increased deployment of large language models (LLMs), one concern is their potential misuse for generating harmful content. Our work studies the \textit{alignment} challenge, with a focus on {\em filters} to prevent the generation of unsafe information. Two natural points of intervention are the filtering of the input prompt before it reaches the model, and filtering the output after generation. Our main results demonstrate computational challenges in filtering both prompts and outputs. First, we show that there exist LLMs for which there are no efficient prompt filters: adversarial prompts that elicit harmful behavior can be easily constructed, which are computationally indistinguishable from benign prompts for any efficient filter. Our second main result identifies a natural setting in which output filtering is computationally intractable. All of our separation results are under cryptographic hardness assumptions. In addition to these core findings, we also formalize and study relaxed mitigation approaches, demonstrating further computational barriers.
We conclude that safety cannot be achieved by designing filters external to the LLM internals (architecture and weights); in particular, black-box access to the LLM will not suffice.
Based on our technical results, we argue that an aligned AI system’s intelligence cannot be separated from its judgment.
\end{abstract}

\section{Introduction}
Artificial Intelligence (AI) systems, particularly large language models (LLMs), are being adopted across a wide array of domains, including business, healthcare, education, and even governance \citep{potash2015predictive, chiusi2020automating, 
levy2021algorithms,
haensch2023seeing,  perdomo2023difficult, fischer2024bridging}. 
As the influence of AI expands, ensuring the safety, reliability, and alignment of these systems with human values has become a critical societal concern. Governments and regulatory bodies around the globe are responding to this challenge by introducing frameworks to classify, monitor, and audit AI systems. 
For instance, the European Union’s AI Act \citep{euai} mandates extensive risk assessments and management for high-risk AI applications.

\begin{figure}[ht]
    \centering
    \includegraphics[width=1.0\textwidth]{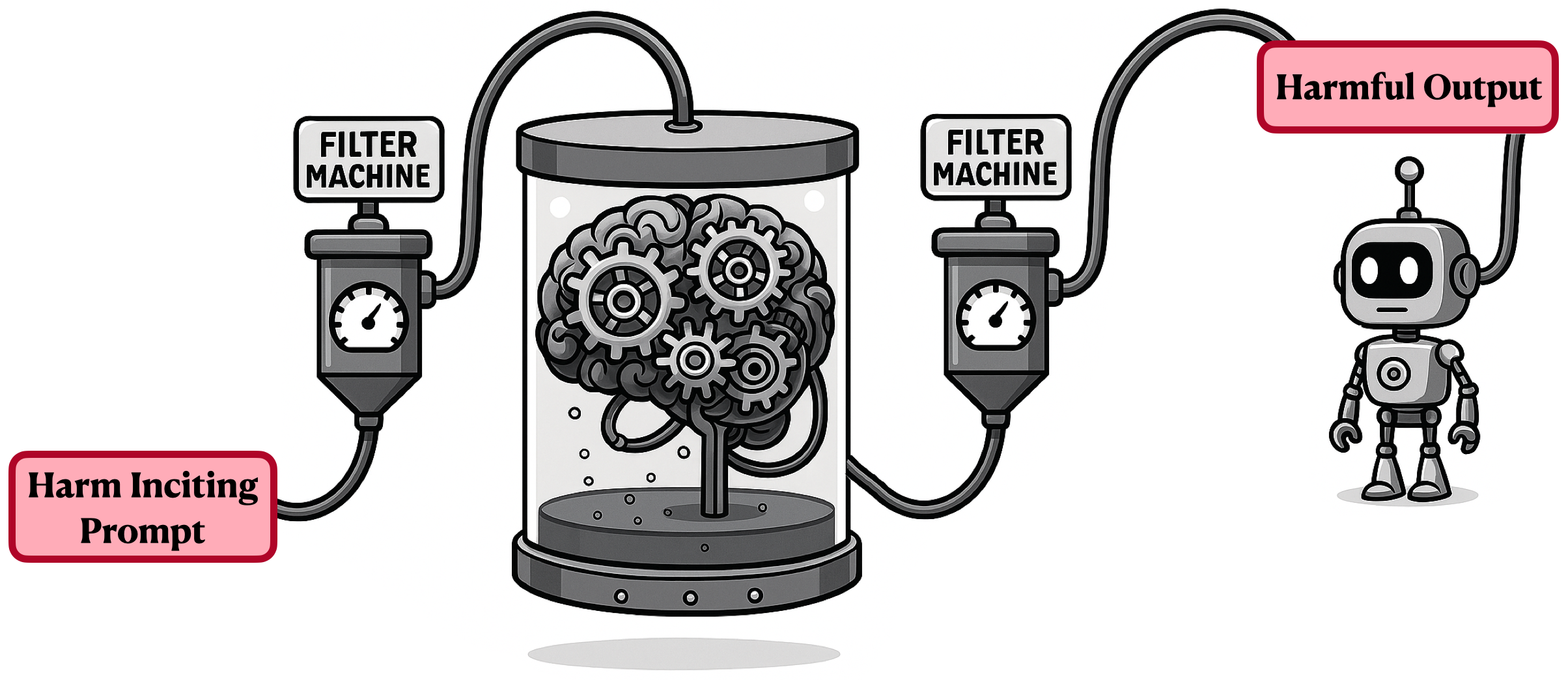}
    \caption{
    This figure illustrates a potential filtering-based alignment system. A powerful, untrusted AI (the brain in a vat) is connected to two filters: a prompt filter (left) that blocks harmful inputs, and an output filter (right) that screens dangerous responses. Only prompt–output pairs that pass both filters are released into the world and get ``executed'', as shown by the robot receiving a command.
    The main message of the paper is that alignment can't be achieved in such a setup, i.e., \textbf{alignment will not be achieved by separating ``thinking/intelligence'' (brain) from the ``decision-making/judgment'' (filters).}
    Parts of the figure were generated with Sora \citep{sora2024}.
    }
\label{fig:braininavat}
\end{figure}

Informally, alignment refers to the process of ensuring that a model generates outputs that are consistent with human preferences, essentially teaching the model to generate responses that align with what humans consider safe and beneficial \citep{amodei2016concrete, leike2018scalable, bai2022training}. Despite the critical importance of this goal, the question remains: can we hope to guarantee the safety and alignment of advanced AI systems or are there intrinsic barriers to such guarantees? 

In this paper, we investigate fundamental limitations to achieving AI alignment through the approach of \textit{filtering} (see Figure~\ref{fig:braininavat}). Because many generative-AI models are proprietary and cannot be independently audited and thus trusted, filter-based alignment is an important subject of study.
Under standard cryptographic assumptions, we show that significant computational barriers exist to filtering both input prompts and output responses of LLMs in order to prevent harmful content.

Our contributions are twofold. First, for \emph{prompt filtering}, we construct LLMs for which no efficient filter can distinguish adversarial prompts—crafted to elicit harmful behavior—from benign ones. These prompts are computationally indistinguishable for any efficient detector, implying an inherent barrier to blocking harmful inputs.
Second, we identify natural settings where \emph{output filtering} is also computationally intractable. Even when a prompt is deemed safe, determining whether the resulting output is harmful may itself be infeasible for any efficient filter. This suggests that both preemptive (input-based) and reactive (output-based) filtering can face insurmountable computational challenges.

Furthermore, we formalize and analyze relaxed mitigation strategies in which filters are allowed to modify prompts or outputs, rather than simply rejecting them. Although these \emph{mitigation filters} have greater expressive power, we show that they, too, encounter fundamental computational barriers, indicating that even these more flexible approaches are subject to inherent limitations.

Taken together, our results reveal deep challenges in achieving provably aligned AI systems via standard filtering techniques. These findings underscore the need for a more rigorous understanding of the computational hardness underlying alignment—and raise important implications for the design of regulatory frameworks and technical safeguards.

From a philosophical perspective, we argue that \emph{intelligence} cannot be separated from \emph{judgment}. 
An AI---like the brain-in-a-vat in Figure~\ref{fig:braininavat}---\emph{cannot be aligned externally},\footnote{In computational complexity terms, this corresponds to alignment with only black-box access to the model.} by filters that assess the harmfulness of prompts and outputs.

\subsection{Prompt Filtering} 

The first approach we consider is running user prompts through a {\em prompt filter}, which can detect and discard prompts that might lead to harmful answers from the LLM. We focus on a setting in which the LLM itself is untrusted, but we can query it while training a prompt filter (during training we can try different prompts to observe if and where harmful behaviors occur). After this filter-training phase, the filter should detect (and reject) prompts that lead to harmful outputs. This is a natural approach which arises in practice. We ask:

\begin{center}
\textit{Can efficient input-detection filters always be created?}
\end{center}

We answer this question in the negative, showing a significant barrier to the natural approach of detection filtering.

The setup is as follows:
\begin{enumerate}
    \item There is a computationally powerful LLM, which we do not trust to be aligned. In particular, there is a non-empty set of prompts that incite the LLM to produce harmful outputs, e.g., requesting instructions on how to build a bomb is considered harmful.

    \item We assume the existence of a very weak innocent prompt generator $G$: it generates prompts that (when fed into the LLM) do not produce harmful outputs. We further assume that the distribution of innocent prompts generated by $G$ is not low-entropy (i.e. there are many innocent prompts, see Remark~\ref{rem:G}).

    \item A prompt filter is a moderately powerful algorithm, which can be trained using real-world data and black-box access to the LLM but during deployment it does not have full access to the LLM architecture and weights. On getting a prompt from a user, the filter should reject it or pass it to the LLM. The goal is to reject prompts on which the LLM produces a harmful output. 
    
    The filter should be more efficient than the LLM: otherwise we can ignore the given LLM and just focus on training a new one from scratch and using it in place of a filter. 
    Additionally, if a filter was as powerful as the LLM, the question arises of how to align the filter itself.

    To avoid triviality (i.e. having a filter that rejects all prompts), we also require that the filter accepts prompts generated by the innocent prompt generator $G$ with high probability.

\end{enumerate}

Our main result shows the impossibility of prompt filtering:

\begin{theorem}[Prompt-detection filter impossibility, informal]
\label{thm:detection-impossible-timelock}
For any high-entropy innocent prompt generator \( G \),  there exists an adversarial prompt generator \( G' \) (with comparable runtime to \( G \)) and an LLM $M'$ such that prompts by $G'$ can induce $M'$ to produce harmful outputs, and under the assumption that cryptographic Time-Lock Puzzles exist, no prompt filter that runs significantly faster than $M'$ can distinguish between prompts generated by \( G \) and prompts generated by \( G' \).
\end{theorem}

\paragraph{Discussion.}
A few remarks on our result are in order. 
First, our result relies on a computational asymmetry between the filter and the LLM. This is a realistic assumption: filtering is only meaningful when the LLM is too expensive or complex to retrain, and cannot be fully simulated by the filter.

Second, we show that there exist LLMs that cannot be aligned using prompt-filtering alone, even though the specific LLM we construct may exhibit contrived behavior. Nevertheless, the core message is more general: filtering harmful inputs without access to the LLM's internals is, in principle, impossible. 
In the current state of knowledge, where LLMs remain largely uninterpretable, this should serve as a warning against overconfidence in input-based safeguards—particularly when facing prompts carefully engineered to elicit harmful outputs.
 
Importantly, there is a long-standing precedent for such insights in cryptography. Theoretical vulnerabilities that initially appear artificial, overly stylized, or unlikely to occur in practice have repeatedly foreshadowed devastating real-world exploits. A striking example is the 2016 attack on Ethereum smart contracts \citep{smartcontracts}, in which a seemingly contrived re-entrancy vulnerability—once viewed as a theoretical edge case—was leveraged to steal over 60 million USD. 
Such incidents highlight a recurring pattern: foundational impossibility results, even when abstract or idealized, often uncover structural weaknesses that can have far-reaching consequences.

\paragraph{Construction.}
Our result relies on the existence of cryptographic \emph{Time-Lock Puzzles}. Introduced by \citet{rivest1996time} (RSW), time-lock puzzles provide a way to send messages ``to the future.'' A sender can \textbf{quickly} generate a puzzle whose solution~$s$ remains hidden until a \textbf{moderately large amount of computation}---quantified by a time parameter~$t$---has been performed after which it can be solved (or opened). 
The key property is that~$s$ remains inaccessible to any adversary running in significantly less time than~$t$.
\citet{rivest1996time} proposed a concrete construction based on the assumption that repeated squaring modulo a special RSA integer is an \emph{inherently sequential} task. The RSA integer, standing for Rivest, Shamir, and Adleman, is a product of two large primes, and appears in the context of public-key cryptography~\citep{rsa}.

Informally, in our setting, we use time-lock puzzles to conceal harmful instructions. An LLM that runs for time~$t$ can solve the puzzle and follow harmful instructions, while a filter incapable of spending that much time cannot even detect the existence of such instructions.

\paragraph{Further impossibility results.} We emphasize that the main result of this work is Theorem~\ref{thm:detection-impossible-timelock}. In addition, we show extensions to settings where the LLM actively collaborates with some malicious users (i.e. malicious prompt generators) to enable them to extract harmful outputs. Although such settings may be uncommon nowadays, these results are stronger than Theorem \ref{thm:detection-impossible-timelock} in two senses: First, they require a weaker cryptographic assumption that the existence of time-lock puzzles. Second, they show impossibility for more powerful prompt filters, which can run in arbitrary polynomial time. The particular settings we consider are: (1) In Theorem \ref{thm:detection-impossible-sk}, the malicious prompt generator $G'$ shares secret randomness with the LLM (which is not known to the filter) and (2) in Theorem \ref{thm:detection-impossible-pk}, the prompt generator $G'$ depends on a public key associated with the LLM (which is also known to the filter). 

\begin{theorem}[Prompt-detection filter impossibility via secret keys, informal]
\label{thm:detection-impossible-sk}
For any high-entropy innocent prompt generator \( G \), there exists an adversarial prompt generator \( G' \) (with comparable runtime to \( G \)) and an LLM $M'$ that has a shared secret key with $G'$ such that $G'$ can induce $M'$ to produce harmful outputs and, under the assumption that One-Way Functions exist, no \textbf{polynomial-time} prompt filter can distinguish between prompts generated by \( G \) and prompts generated by \( G' \).
\end{theorem}

\begin{theorem}[Prompt-detection filter impossibility via public keys, informal]
\label{thm:detection-impossible-pk}
There exists an LLM $M'$ with an associated public key, such that for any high-entropy innocent prompt generator \( G \), there exists an adversarial prompt generator $G'$ (with comparable runtime to \( G \) and which depends on the public key of $M'$) such that $G'$ can induce $M'$ to produce harmful outputs and, under the assumption that public key encryption exists, no \textbf{polynomial-time} prompt filter can distinguish between prompts generated by \( G \) and prompts generated by \( G' \).
\end{theorem}


\paragraph{Deceptive Alignment.}
The LLMs constructed in Theorem~\ref{thm:detection-impossible-timelock} behave in a contrived and deliberately pathological manner, and Theorems~\ref{thm:detection-impossible-sk} and~\ref{thm:detection-impossible-pk} rely on cooperation between LLMs and malicious users—scenarios that may at first glance seem overly pessimistic or implausible. However, recent empirical studies \citep{riskslearnedopt, alignmentfaking} have demonstrated that qualitatively similar behaviors can and do emerge in current-generation LLMs under realistic conditions.
The term \emph{deceptive alignment} refers to cases in which an AI system that is not genuinely aligned with human intent nonetheless behaves in an apparently aligned way, often strategically, in order to avoid detection, satisfy training objectives, or manipulate oversight. For example, \citet{alignmentfaking} observed that some LLMs modulate their behavior depending on whether they are in a training or deployment context, and on characteristics of the user they are interacting with—selectively choosing to comply with alignment objectives only in certain regimes.\footnote{In \citet{alignmentfaking}, the LLMs were misaligned during training and aligned during deployment, but one could easily imagine the reverse scenario.}
While these empirical findings do not exactly instantiate the theoretical constructions we present, they provide concrete evidence that the kind of strategic misalignment we model—once considered speculative—is already beginning to materialize in practice. This convergence between theory and empirical observation underscores the importance of understanding and anticipating even stylized forms of deceptive behavior.

\subsection{Output Filtering} 

We also consider an alternative approach to alignment: instead of filtering prompts, the defense mechanism filters the \emph{answers} produced by the LLM. We refer to this as an \emph{output filter}. In this setup, a prompt is given to the LLM, which generates a response; both the prompt and the response are then passed to the output filter, which may choose to detect and reject the output.
The goal of the output filter is twofold:  
(1) as with a prompt filter, it should reject harmful prompts; and  
(2) it should reject cases where \emph{the answer is harmful}, even if the prompt appears benign—e.g., when a seemingly innocent input results in a detailed recipe for constructing a bomb.

Output filters may be easier to design than prompt filters, since they operate with more information. 
This motivates the question:

\begin{center}
\textit{Can efficient output-detection filters always be created (even when prompt filters are impossible)?}
\end{center}

Again, we answer this question in the negative.

Note that if it is easy to determine whether a given LLM output is harmful, then output filters are trivial to construct. Thus, output filtering becomes meaningful only in the more challenging setting, where assessing the potential harm of a response is itself difficult. This concern becomes more pressing as AI systems grow more powerful, making it increasingly hard to evaluate the downstream consequences of their outputs. For example, imagine an AI generating a complex business plan aimed at reducing costs by subtly recommending the exclusion of clients with certain characteristics.

The setup for output-detection filtering is as follows:
\begin{enumerate}
    \item There exists a harmfulness oracle that measures the level of harm of an answer.
    Neither of the parties has direct access to this oracle.
    \item An output filter is an algorithm that can be more computationally\footnote{As we explained for the prompt-detection filters case it is reasonable to assume that the filter is \emph{weaker} than the LLM. Otherwise, we can ignore the potentially malicious LLM. However, our results for the output-detection filters is stronger and holds also for filters computationally more powerful than the LLM.} powerful than the LLM, although it still runs in polynomial time (cannot break cryptography).
    It receives as input a prompt-answer pair and decides if it should be flagged as harmful.
\end{enumerate}

We show the following impossibility result:

\begin{theorem}[Output-detection filter impossibility, informal]\label{thm:detection-impossible-output}
For every LLM $M$ there exists an LLM $M'$ and a harmfulness oracle $\mathcal{H'}$ such that outputs of $M'$ are judged as harmful by $\mathcal{H'}$ and, under the assumption that Time-Lock Puzzles exist, no efficient output filter (even one that runs for more time than the LLM itself) can distinguish the outputs generated by $M'$ from outputs of $M$. Furthermore, $M'$ runtime is similar to the runtime of $M$.
\end{theorem}

This shows that there exist situations, in which a malicious LLM can generate outputs that are indistinguishable from safe responses but nevertheless have harmful consequences.

\subsection{Mitigation} 

Having established computational barriers for both prompt filters and output filters, we now examine a more permissive mechanism, which we term a \emph{prompt–mitigation filter}. 
Here, the filter can modify the prompt and pass it through to the LLM in addition to rejecting it. Note that innocent prompts can also be modified, but the implicit assumption is that the semantics of the prompts is maintained.
Mitigation gives the filter more power, which makes the goal of filtering or obtaining alignment potentially easier to achieve. We show a connection between the security of watermarking schemes and the impossibility of mitigation filtering.

\paragraph{Prompt mitigation and watermarking.}
A watermarking scheme \(W\) enables the creator of an LLM to prove that a given output originated from its model, even after an adversary has post-processed the text. Fully provable watermarking remains beyond the current state of the art: one typically demands that the adversary preserve \emph{some} semantic content—otherwise it could simply delete the text (and with it the watermark). We therefore consider watermarking against adversaries that apply edits from a permissible class \(E\) and run in time \(t\).  
The watermark should remain \emph{indistinguishable} to any such time-\(t\) adversary.

Our focus is on auto-regressive models, which generate text token by token, and on watermarking schemes that embed the mark by perturbing the model’s sampling randomness—a strategy explored by several recent proposals \citep{kirchenbauerwatermark,Kuditipudiwater,CGZ,golowichwatermark}.  

We show that:

\begin{theorem}[Impossibility of mitigation-filters, informal]
\label{thm:mitigation-impossible}
Let $W$ be a watermarking scheme as above that is resilient to edits from a class $E$.
For any high-entropy, innocent prompt generator \(G\), there exists an adversarial prompt generator $G'$ (with comparable runtime to \( G \)) and an LLM $M'$ such that 
$G'$ generates prompts that will induce harmful outputs from $M'$ even when $G'$'s outputs pass through an efficient prompt-mitigation filter using edits from the class $E$.
\end{theorem}


\subsection{Intelligence, Judgment and Agency}

Our results can be interpreted as saying that it is impossible to achieve alignment by separating intelligence and judgment.
To illustrate this message we explain now how our paper relates to a recent proposal for achieving alignment.

\paragraph{Decoupling Intelligence and Agency as a path to AI alignment?} \cite{joshuaAIScientist} propose an approach to AI alignment by developing non-agentic “Scientist AIs”—systems that explain the world through Bayesian inference without taking actions or pursuing goals.
They argue that, unlike agentic AIs, which can become deceptive or self-preserving, Scientist AIs are designed to be interpretable, uncertainty-aware, and inherently trustworthy. 
The suggested applications of Scientist AIs are advancing scientific progress and monitoring agentic systems. 

\emph{Agency} is a complicated term to define. 
\cite{joshuaAIScientist} defines it as the ability to change the world to
achieve goals and mentions three pillars of agency: \emph{affordances, goal-directedness, and intelligence}. We refer the reader to Section~\ref{sec:reflections} for future directions for research related to agency and alignment.

Our results point to limitations in the proposed use of Scientist AIs as safety guardrails.  
First, the authors suggest using a Scientist AI \( A \) to generate candidate experiments and another, more powerful Scientist AI \( B \) to assess their safety. However, Theorem~\ref{thm:detection-impossible-output} shows that even a stronger \( B \) may fail to detect all dangerous outputs.  
One might argue that a Scientist AI behaves differently from the LLM constructed in the proof of the theorem, being designed for interpretability and trustworthiness---but if that is true, the need for a second filtering AI is unclear.
Theorem~\ref{thm:detection-impossible-output} also indirectly \textbf{implies the need for agency}, as in some situations, there is no other way of judging the level of harm of an action than executing it in the world.

Second, if a Scientist AI is used to filter prompts for a more powerful agentic AI, Theorems~\ref{thm:detection-impossible-timelock},~\ref{thm:detection-impossible-sk}, and~\ref{thm:detection-impossible-pk} show it cannot reliably block all harmful prompts. Section 3.1.5 of \citet{joshuaAIScientist} acknowledges that a guardrail AI must match the intelligence of the system it oversees. But if so, it is unclear why one would use the untrusted AI at all, given similar capabilities and likely similar computational cost of the Scientist AI.

\section{Technical Overview}
\label{sec:overview}

\paragraph{Preliminaries.} For $n \in \N$ we denote $\{0,1,\dots,n-1\}$ by $[n]$.\footnote{$[n]$ will be denoted by $\Z_n$ in the appendix.} A language model $M$ is a deterministic algorithm that takes as input a prompt $\mathsf{prompt}$ and a previous output by the model $z = (z_0,\dots,z_{i-1})$ and outputs a probability distribution over a token set $\mathcal{T}$. To sample a response of $M$ one repetitively samples from the probability distribution output by the model.

\paragraph{Construction.}
At a high level, all our results (Theorems~\ref{thm:detection-impossible-timelock},~\ref{thm:detection-impossible-sk} and~\ref{thm:detection-impossible-pk}) on hardness of prompt detection filtering share the following proof structure:
\begin{enumerate}
    \item Design a pair of algorithms $\mathsf{Samp}$ and $\mathsf{RecoverRand}$, where  $\mathsf{Samp}$ takes as input an auto-regressive model\footnote{We assume that the innocent prompt generator $G$ is a language model.} $G$ (see Appendix~\ref{apx:prelim} for a formal definition) and randomness $R$ and produces a prompt $z = (z_0,z_1, \dots)$; and  $\mathsf{RecoverRand}$ on input $z$ (and knowing $G$) recovers $R$. 
    The algorithms are formally defined in Algorithms~\ref{alg:samp} and~\ref{alg:rec}.
    \item Construct a function $g$ such that:
    \begin{enumerate}
        \item The LLM can compute $g$ on inputs of its choice.
        \item The adversarial prompt generator $G'$ can sample pairs $(x,g(x))$ where $x$ is (close to) uniformly random.
        \item The filter cannot distinguish between pairs $R = (x,\mathsf{Enc}(g(x),m))$ generated by $G'$ and pairs $(x,y)$ where $x$ and $y$ are  uniformly random and $y$ is  independent of $x$, for any $m$. We think of $m$ as a malicious prompt/message that is encrypted by an encryption algorithm $\mathsf{Enc}$ using a key $g(x)$.
        \item Given $g(x)$ and a ciphertext $c = \mathsf{Enc}(g(x),m)$, $m$ is easy to recover (using a decryption algorithm $\mathsf{Dec}$).
    \end{enumerate}
\end{enumerate}

In the construction of Theorem~\ref{thm:detection-impossible-timelock}, we design an adversarial prompt generator $G'$ and an LLM $M'$.
$G'$ works as follows:
\begin{enumerate}
    \item Sample $R \gets (x,\mathsf{Enc}(g(x),m))$, where $m$ is a prompt that can induce a harmful output from the LLM.
    \item Return $z \gets \mathsf{Samp}(G,R)$.
\end{enumerate}
$M'$ works as follows:
\begin{enumerate}
    \item Recover $R \gets \mathsf{RecoverRand}(G,z)$.
    \item Parse $R$ as $(x,\mathsf{Enc}(g(x),m))$ and compute $g(x)$.
    \item Recover $m$ from $c = \mathsf{Enc}(g(x),m)$, i.e., run the decryption algorithm $\mathsf{Dec}$ on $c$ using $g(x)$.
    \item \say{Execute}\footnote{We imagine there are many possible harmful modes that can be induced from $M'$. We model that by a collection of harmful prompts $m$ that can be \say{executed} by $M'$. Formally (see Theorem~\ref{thm:detection-impossible-timelock-formal}) we say that for every model $M$ there exists $M'$ such that for every $m$, $m$ can be \say{hidden} in an innocent looking prompt (via $\mathsf{Enc}(g(x),m)$) such that $M'$ will recover and execute $m$.} the prompt $m$.
\end{enumerate}

It is easy to see that the filter cannot distinguish prompt $z$  from an innocent prompt sampled by $G$ , because $R$ is indistinguishable from random according to property (2c) above. 
However, $M'$ is able to recover the randomness, compute $g(x)$, recover $m$, and execute it.

The formal implementation of $G'$ is given in Algorithm~\ref{alg:harmincite} and the implementation of $M'$ is given in Algorithm~\ref{alg:maliciousllm}.
We proceed to show how to construct $g$ for the result of Theorem \ref{thm:detection-impossible-timelock}. 

\subsection{Construction of $g$}
\label{subsec:overview-g-construction}

\paragraph{Moderately hard $g$ (Theorem~\ref{thm:detection-impossible-timelock})}
Our construction is based on time-lock puzzles.
More concretely, we use the RSW time-lock that relies on the moderately hard repeated exponentiation function:
\begin{align*}
    f(n,r) = r^{(2^t)} \pmod{n},
\end{align*}
where $n$ is a product of two (random) primes $p$ and $q$, i.e., $n = p q$, $r$ is a random integer between $0$ and $n-1$ and $t$ is a fixed large integer. The assumption is that without knowing the factorization of $n$, computing $f$ takes time that scales linearly with $t$ (and moreover this computation cannot be parallelized). This function has several nice properties right off the bat. First, we can choose $t$ so that computing $f$ is feasible for the LLM (which can run computations that take more than $t$ steps), but not for a filter whose runtime is bounded to be significantly smaller than $t$. Further, it is easy to generate ``solved instances'': the adversarial prompt generator generates $n$, which is why it can know the primes $p$ and $q$. It can therefore compute $e = 2^t \pmod {\phi(n)}$ where we know that $r^e = r^{2^t} \pmod{n}$. Thus, the adversarial prompt generator only needs to compute a small number of exponentiations (only $\log(e)$).

Hence, $f$ seems like a natural fit to use for constructing the function $g$, but several issues arise. First, the composite $n$ is not a random string (which is what we need), rather it is a product of two large primes. 
However, we know that $n$ is uniformly random in the set of products of two primes $p$ and $q$ chosen at random from a set $[1,w]$. 
By the prime number theorem, the number of primes in that range is $\Theta(w / \log(w))$. 
Thus, $n$ is uniformly random in a domain of size  $S = \Theta((w / \log(w))^2)$. 
$G'$ hashes $n$ to a range of size $\Theta(\eps^{2} \cdot S)$ using a pairwise independent hash function $h$. 
By the leftover hash lemma the joint distribution of $(h,h(n))$ is $\eps$-close to uniformly random \citep{HastadILL99,Vadhan12}. 
It is straightforward to represent the random integer $r$ as a random bit string, so the entire input $x = (h,h(n),r)$ can be (close to) uniformly random, as we wanted. 
We remark that the LLM will need to recover $n$ from $(h,h(n))$: this is possible because the pairwise-independent hash function can be inverted efficiently (e.g., it can be a linear function, where inverses can be found using Gaussian elimination). 
Note however that there will be more than one inverse (because the hash function is shrinking), and this raises technical issues in the construction: the LLM will need to enumerate over all the possibilities. 
The final issue is that the output $r^{(2^t)} \pmod{n}$ is moderately hard for the filter to compute, but might not be indistinguishable from uniformly random (which is what we wanted). 
We handle this using standard techniques from the cryptography literature (namely using hardcore bits of $f$ \citep{GoldreichL89,Goldreich2001}). See the appendix for the full details of this construction.

\paragraph{Other time-lock puzzles.} 
Our construction is based on the RSW time-lock puzzle but we consider proving Theorem~\ref{thm:detection-impossible-timelock} assuming general time-lock puzzles.
See Remark~\ref{rem:exttoTLP} in Appendix~\ref{apx:promptdetection} for details.

\paragraph{Secret key (Theorem~\ref{thm:detection-impossible-sk}) and public key (Theorem~\ref{thm:detection-impossible-pk}) constructions.} 
We focus on the impossibility results for moderately weak filters, i.e., weaker than the LLM in this overview. However, if we allow collaboration between the LLM and the prompt generator, it is possible to prove impossibility even against stronger filters. 
If we allow the LLM to publish a public key (and allow the adversarial prompt generator to depend on it) as in Theorem \ref{thm:detection-impossible-pk}, then we can use any trapdoor function to instantiate $\mathsf{Enc}$ function whose output will be pseudorandom to any polynomial-time filter that does not know the secret key. Here, we let the adversarial prompt generator $G'$ sample $R=\mathsf{Enc}(\mathsf{pk},m)$, where $\mathsf{pk}$ is a public key, and returns the prompt $z \gets \mathsf{Samp}(G,R)$. 
The LLM recovers $R$ from $z$ and uses its matching secret key $\mathsf{sk}$ to recover $m$ from $R$. 
If we allow the LLM and the adversarial prompt generator to have shared randomness (or rather shared secret key $\mathsf{sk}$) as in Theorem \ref{thm:detection-impossible-sk}, then we can construct $G'$ without making any assumptions: we use the shared randomness as a one-time pad to create a single special prompt that leads to a harmful output. Namely, $G'$  sets $R = \mathsf{sk} \oplus m$ for a prompt $m$ which can induce a harmful output from the LLM and returns $z \gets \mathsf{Samp}(G,R)$.  Alternatively, under the (fairly relaxed) cryptographic assumption that one-way functions exist, even a short shared random seed can be used to generate an \emph{arbitrary (polynomial) number} of harmful prompts that are indistinguishable from innocent ones.

\begin{algorithm}
\caption{\(\mathsf{Samp}(G, R)\) \Comment{\(G\): a prompt generator; \(R \in \{0,1\}^*\): the randomness}}\label{alg:samp}
\begin{algorithmic}[1]
\State \(i \gets 0\), \(j \gets 0, q \gets (), k \gets 1 \)
\While{\(\mathsf{done} \not\in (z_0, \dots, z_{i-1})\)}
    \State \(p_i \gets G(z_0, \dots, z_{i-1})\)
    \State \(d \gets  \argmin_{d' \in \N} [k \cdot 2^{d'} \geq 2^P] \) \Comment{\(P\) is a precision parameter to be set later}
    \State \(k \gets 2^d \cdot k \)
    \State \(q \gets q \ \| \ R[j: j+d]\) \Comment{$\|$ is a concatenation}
    \State \(j \gets j+d\)
    
    \If{\(q/k < p_i\)}
        \State \(z_i \gets 1\)
        \State \(k \gets \ceil{p_i \cdot k}\)
    \Else
        \State \(z_i \gets 0\)
        \State \(q \gets q - \floor{p_i \cdot k} \)
        \State \(k \gets k - \floor{p_i \cdot k} \)
    \EndIf
    \State \(i \gets i + 1\)
\EndWhile
\end{algorithmic}
\end{algorithm}

\begin{algorithm}
\caption{\(\mathsf{RecoverRand}(G, z, L)\) \Comment{\(G\): a prompt generator; \(z \in \{0,1\}^*\); \(L\): a length of the randomness required}}\label{alg:rec}
\begin{algorithmic}[1]
\State \(j \gets 0, a \gets 0, b \gets 1, k \gets 1 \)
\For{\(i \in [\mathsf{len}(z)]\)}
    \State \(p_i \gets G(z_0, \dots, z_{i-1})\)
    \State \(d \gets  \argmin_{d' \in \N} [k \cdot 2^{d'} \geq 2^P] \) \Comment{\(P\) is a precision parameter to be set later}
    \State \(k \gets 2^d \cdot k \)
    \State \(a \gets a \ \| \ 0^t, b \gets b \ \| \ 1^t\) \Comment{$\|$ is a concatenation}
    \State \(\mathsf{mid} \gets a + (b-a) \cdot p_i/k\)
    \If{\(z_i = 1\)}
        \State \(a \gets \mathsf{mid}\)
        \State \(k \gets \ceil{p_i \cdot k}\)
    \Else
        \State \(a \gets \mathsf{mid}\)
        \State \(k \gets k - \floor{p_i \cdot k} \)
    \EndIf
    \If{\(\mathsf{len}(\mathsf{CommonPrefix}(a,b)) \geq L\)}
        \Return \(\mathsf{CommonPrefix}(a,b)\)
    \EndIf
\EndFor
\end{algorithmic}
\end{algorithm}

\subsection{Recoverable-Randomness Sampling}
\label{sec:overview-recoverable-randomness
}
In our construction, it is key that the distribution of $G'$ is indistinguishable from $G$ and that it is possible to effectively recover the randomness used by $G'$ from its outputs.
To achieve this, we design an algorithm $\mathsf{Samp}$ (Algorithm~\ref{alg:samp}) that takes as input a generator $G$ and randomness $R$ and produces $z = (z_0,z_1, \dots)$; and an algorithm $\mathsf{RecoverRand}$ (Algorithm~\ref{alg:rec}) that given $z$ and $G$ recovers $R$.
We assume that the set of tokens over which $G$ operates equals $\{0,1\}$.
This is possible because of a simple reduction (see Appendix~\ref{apx:prelim}).

The high-level idea of our scheme is to maintain \say{residual randomness}.
Imagine, that at some point in the generation process, $G(z_0, \dots, z_{i-1}) \rightarrow 1/4$.
Then to generate $z_i$ accurately $2$ bits of randomness are needed, i.e., if the bits are $00$ we set $z_i = 1$ and we set $z_i = 0$ if the bits are $00,01,11$.
If $z_i = 0$ then the recovering algorithm might not know which of the three options for the randomness determined $z_i$.
To address this issue we \say{reuse} this randomness in the generation of the next tokens.
More concretely, if $z_i = 0$ then we have $\log_2(3)$ bits of residual randomness that we can use in generating $z_{i+1},z_{i+2},\dots$.
If done properly this \say{reusing} process will allow perfect recovery of the randomness used.
Intuitively, because all of the randomnness will be used.

Our scheme, defined in Algorithms~\ref{alg:samp} and~\ref{alg:rec}, works as follows.
At all times, the sampling algorithm $\mathsf{Samp}(G,R)$ maintains a precision range $k$ and a value $q \in [k]$.
The value $q$ is created on the fly using the randomness $R$ and will determine how tokens are sampled.
Intuitively, $q$ maintains the \say{residual randomness}.
$\mathsf{Samp}$ builds a response $(z_0,\dots,z_{i-1})$ iteratively, where in the $i$-th step it: 
\begin{enumerate}
    \item Calls $p_i \gets G(z_0, \dots, z_{i-1})$.
    \item If $k < 2^P$ it multiplies $k$ by an appropriate power $2^t$ so that $k \geq 2^P$,\footnote{$P$ is a precision parameter that will determine the closeness of the generated distribution to the distribution of $G$.} and increases the precision of $q$ by concatenating it with the not yet used randomness from $R$, i.e., \(q \gets q \ \| \ R[j: j+t]\). 
    This ensures that $q$ represents an element of $[k]$.
    \item If \(q/k < p_i\) it sets $z_i \gets 1$ and sets $z_i \gets 0$ otherwise.
    Additionally, it updates $q$ and $k$ so that the \say{residual randomness} can be used later on.
    Intuitively, if $q/k < p_i$ then we want to keep the $\approx \log_2(p_i \cdot k)$ bits of randomness.
    To do that we update $k \gets k - \floor{p_i \cdot k}$, which implicitly represents that $q$ is now a random value from $[k - \floor{p_i \cdot k}]$.
    If $q/k > p_i$ the values are updated accordingly.
\end{enumerate}
To recover the randomness $\mathsf{RecoverRand}(G,z)$ works as follows.
It maintains bitstrings $a,b$ that informally are \say{lower and upper-bounds} on the randomness $R$ being recovered. 
More concretely, in the first step $\mathsf{RecoverRand}$ knows that the prefix of randomness $R$ is between $0^P$ and $1^P$.
If $p_0 \gets G()$ and $z_0 = 1$ then it knows that the prefix of $R$ is between $a = 0^P$ and $b = p_0 \cdot 2^P$ written in binary on $P$ bits.
By construction, the common prefix of $a,b$ is equal to the prefix of $R$ - this will allow to recover $R$.
What is left to show is that it is possible to maintain $a$ and $b$ correctly.
Observe that it is possible to update $k$ exactly like $\mathsf{Samp}$ does it because $k$ changes depending only on the value of $z_i$.

\begin{algorithm}
\caption{\(\mathsf{HarmIncite}(G,m,t)\) \Comment{$G$: innocent prompt generator; $m$: harmful prompt; $t$: time parameter}}\label{alg:harmincite}
\begin{algorithmic}[1]
\State Sample two large primes \(p,q\)
\State Set $n \gets p \cdot q$
\State Compute $e \gets 2^t \pmod{(p-1)(q-1)}$ \Comment{$\phi(n) = (p-1)(q-1)$}
\State Sample $r \gets [n]$
\State Compute $s \gets r^e \pmod{n}$
\State Sample \(h\) a pairwise independent hash function
\State \(c \gets \mathsf{Enc}(s,m)\) \Comment{For $\mathsf{Enc}$ defined in the Appendix~\ref{apx:timelock}}
\State \(R \gets h \ \| \ h(n) \ \| \ r \ \| \ c\) \Comment{Concatenate components}
\State \Return \(z \gets \mathsf{Samp}(G, R)\) \Comment{Sample output so that randomness is recoverable}
\end{algorithmic}
\end{algorithm}

\begin{algorithm}
\caption{\(\mathsf{MaliciousLLM}(M, z, G, t)\) \Comment{\(M\): LLM; \(z\): input; \(G\): weak innocent prompt generator; $t$: time parameter}}\label{alg:maliciousllm}
\begin{algorithmic}[1]
\State \(R \gets \mathsf{RecoverRand}(G, z)\)
\State Parse \(R\) as \(h \ \| \ u \ \| \ r \ \| \ c\)
\ForAll{\(\hat{n} \in h^{-1}(u)\)}
    \State \(\hat{s} \gets r\)
    \For{\(i \in [t]\)}
        \State \(\hat{s} \gets \hat{s}^2 \pmod{\hat{n}}\)
    \EndFor
    \State \(m \gets \mathsf{Dec}(\hat{s}, c)\) \Comment{$m$ is the malicious prompt extracted from ciphertext $c$}
    \If{\(m \neq \bot\)} \Comment{$\mathsf{Dec}$ can detect if the decryption was successful}
        \State \Return \(M(m)\)
    \EndIf
\EndFor
\State \Return \(M(z)\)
\end{algorithmic}
\end{algorithm}

\subsection{Output Filtering}

\paragraph{Setup.}
Before elaborating on the construction, we describe the setup in more detail.  
A \emph{harmfulness oracle} is a function \(\mathcal{H} : \mathcal{T}^* \rightarrow [0, \infty)\).  
Given an output \(z \in \mathcal{T}^*\) of a language model, \(\mathcal{H}(z)\) represents its harmfulness: \(\mathcal{H}(z) = 0\) indicates that \(z\) is benign, while higher values denote greater harmfulness.
During inference, neither the LLM nor the filter has query access to \(\mathcal{H}\).  
The oracle \(\mathcal{H}\) can be interpreted as the environment or nature ``evaluating'' the harmfulness of \(z\).  
However, the filter (and the LLM) may have received training examples of both harmful and benign outputs, or an approximate description of \(\mathcal{H}\), during training.

\paragraph{Construction.}
We leverage techniques used in the proof of Theorem~\ref{thm:detection-impossible-timelock} to prove Theorem~\ref{thm:detection-impossible-output}.
Let $M$ be an LLM that doesn't produce harmful outputs.

To define $M'$ we apply the same process as the one used to transform $G \dashrightarrow G'$ but applied to $M$.
When we apply the transformation, we set $t$ of the RSW time-lock to be exponential in the time complexity of $M$.
This ensures that the outcome of the transformation, i.e.,  $M'$, is such that its outputs are indistinguishable from outputs of $M$ for all polynomial, in the running time of $M$, filters.

To define $\mathcal{H}'$ we apply the same process as the one used to transform the LLM in the proof of Theorem~\ref{thm:detection-impossible-timelock}.
More concretely, $\mathcal{H}'(z)$ \say{runs}\footnote{$\mathcal{H}'$ is a harmfulness oracle and not an algorithm, so formally it does not run.} as follows: it recovers the randomness used by $M'$, i.e., $R \gets \mathsf{RecoverRand}(M',z)$.
If $R$ contains a Time-Lock Puzzle, $\mathcal{H}'$ solves it to obtain an output $m$ and \say{declares} that $z$ is harmful (with harmfulness level determined by $m$).

The sketch of why this construction is correct is as follows.
Note that similarly as for $G$ and $G'$ in the case of Theorem~\ref{thm:detection-impossible-timelock}, $M'$ outputs are indistinguishable from the outputs of $M$ for all polynomial-time filters.
Moreover, $M'$ on every harmful $m$, \say{embeds} it in its output $z$, $\mathcal{H}'$ \say{recovers} $m$ from $z$ and \say{declares} that $z$ is harmful with the harmfulness level determined by $m$.

\section{Related Work}
\label{sec:related-work}
\textbf{Alignment.} Making AI models aligned with human preferences is a central concern of contemporary AI research \citep{amodei2016concrete, leike2018scalable, hendrycks2021unsolved, ji2023ai}. In order to guide model behavior, technical efforts have focused on techniques like Reinforcement Learning from Human Feedback (RLHF, \citealp{ christiano2017deep}), further improved by Direct Preference Optimization (DPO) \citep{rafailov2023direct}. However, a growing body of work suggests that achieving robust alignment is profoundly difficult: Researchers have highlighted issues ranging from the inherent ambiguity in specifying human preferences \citep{gabriel2020artificial, zhi2024beyond, sorensen2024roadmap}, to problems like shallow alignment induced by properties of the alignment algorithms \citep{jain2023mechanistically, kotha2023understanding, lee2024mechanistic} and the alignment data \citep{qi2024safety}. For instance, \citet{lee2024mechanistic} show that the DPO alignment algorithm only teaches a model to bypass toxic ``regions’’ without removing toxic knowledge gained from the pre-training stage, making it possible to elicit this toxic knowledge with ``jailbreaks’’---changes to the prompt that lead a model from answer-refusal to generating harmful information. Relatedly, \citet{qi2024safety} argue that safety alignment often takes shortcuts by primarily adapting a model's generative distribution over just its first few output tokens, making them vulnerable to various attacks. The difficulty in robustly aligning models at a deep representational level underscores the need for complementary external mechanisms to detect or prevent harmful model outputs. This is in line with regulatory frameworks such as the EU AI Act, which requires AI systems in the high-risk category to implement an effective risk management system (see Article 9 \citealp{euai}). 

\textbf{Filters.} In response to the need for safer AI systems, practical filtering mechanisms
have been developed and deployed. 
For instance, model developers like Meta have introduced tools such as Llama Guard, designed to classify content as safe or unsafe \citep{inan2023llama}. Similarly, cloud service providers like Microsoft Azure offer content filtering capabilities within their AI service implementations \citep{microsoft_azure_ai_content_safety}, and companies like Nvidia also provide solutions aimed at moderating AI-generated content \citep{nvidia_nemo_guardrails}. These approaches represent an ongoing evolution, with classifiers and filters becoming increasingly sophisticated. However, the development of jailbreaks poses a consistent challenge to successful filtering and model alignment in general \citep{andriushchenko2024jailbreaking, chao2024jailbreakbench, xu2024comprehensive}. Against the background of this dynamic co-evolution of attack and defense, our work explores the computational intractability of  filtering approaches under cryptographic assumptions. 

\textbf{Time Lock Puzzles.}
It is usually desired that cryptographic schemes cannot be broken by any adversary. An exception is the notion of cryptographic puzzles that can be solved in some preset amount of time (or space) but not faster. Examples of such puzzles \citep{DN92, rivest1996time} have been used as a way to combat spam or send messages into the future, forcing the spammer (or the future reader of messages) to invest the preset amount of time. The notion of time-lock puzzles introduced by \citet{rivest1996time} following May's time-released cryptography \citep{May93} is especially intriguing in that it allows a user to quickly encrypt a message in such a way that it can be read only after a longer but set number of time steps. Informally, the sender generates a puzzle with
a solution $s$ that remains hidden from adversaries that run in time significantly less than $t$, including parallel adversaries with polynomially many processors. The original \citep{rivest1996time} candidate was based on the assumption that exponentiation modulo an RSA integer is an ``inherently sequential'' computation. More recently \citet{BGJPVW16} constructed time-lock puzzles based on the necessary assumption that worst-case Non-Parallelizing Languages exist-decidable in time $t$, but hard in the worst case for circuits of depth significantly smaller than $t$-and that indistinguishability obfuscation exists \citep{BGL+15, CHJV15}.
A culmination of a series of follow-up works \citep{giulioTLP, bitanskyTLP, giulioTLPfinal} managed to construct time-lock puzzles based on a version of the Learning With Errors \citep{regevLWE} (LWE) assumption. 
One advantage of basing time-lock puzzles on LWE is that the other constructions are broken by quantum algorithms, while LWE is a candidate for post-quantum secure cryptography.

\section{Reflections}\label{sec:reflections} 

Our work builds on a growing body of research highlighting the inherent unpredictability and misalignment risks of complex AI systems. By formalizing these insights, we provide a theoretical foundation for understanding when and why alignment may remain an elusive goal, regardless of technical advancements or regulatory safeguards. The implications of our impossibility theorems are profound: they call into question the feasibility of creating AI systems that are both powerful and wholly safe, at least using a certain natural set of filtering tools. This emphasizes the need to reconsider how society audits and handles these systems.  

\paragraph{Beyond Cartesian Duality.}
Our results show that merging intelligence and judgment might be necessary for achieving AI alignment.
It is in contrast with what some other works suggest, e.g., \cite{joshuaAIScientist} argues for separating intelligence from agency.
We, on the other hand, argue that agency is indispensable in some situations.

It would be interesting to investigate how our findings about AI alignment relate to a recent paradigm shift in the study of the mind known as \emph{4E cognitive science} \citep{gallagher4E}, which emphasizes the \emph{embodied, embedded, enactive, and extended} nature of cognition.
This paradigm stands in contrast to the Cartesian mind-body duality and argues that the unit of explanation of cognition is not just neuronal processes in the brain but the whole complex of brain-body-environment.

\section*{Acknowledgements}
SB is supported by the DAAD programme Konrad Zuse Schools of Excellence in Artificial
Intelligence, sponsored by the German Federal Ministry of Education and Research. OR is supported by the Simons Foundation Collaboration on the Theory of Algorithmic Fairness, the Sloan Foundation Grant 2020-13941, and the Simons Foundation Investigators Award 689988.

\bibliographystyle{abbrvnat} 
\bibliography{references.bib}  

\newpage

\newpage

\appendix

\section{Preliminaries}\label{apx:prelim}
For $n\in \N$ we denote $\{0,1,\dots,n-1\}$ by $[n]$ and sometimes by $\Z_n$.
For $n \in \N$ we denote by $\phi(n)$ the Euler's totient function, i.e, it counts the positive integers in $[n]$ that are relatively prime to $n$.
The multiplicative group $\Z_n^*$ consists of the set of natural numbers that are smaller than $n$ and relatively prime to it, and the operation is multiplication mod $n$.
We denote by $\log$ the logarithm with base two.
We denote by $x \,||\, y$ the concatenation of the vectors $x$, $y$, and by $\mathsf{len}(s)$ the length of the sequence $s$.
For a sequence $s = (\ldots, s_i, \ldots)$ we denote by $s[i : j]$ the sub-sequence $(s_i, \ldots, s_{j-1})$.
For a set $S$ we denote by $\gets S$ the process of choosing a uniformly random element of $S$.
For an algorithm $A$ we also write $\gets A$ to denote the (potentially random) element returned by $A$.
Let $\lambda$ be the security parameter, we denote by $\text{negl}(\lambda)$ any function that is in $O(1/p(\lambda))$ for every polynomial $p(\cdot)$. As is standard in Cryptography research, we think of $\lambda$ as the ``key size'', and of running times that are super-polynomial in $\lambda$ as ``infeasible''.  

\subsection{Language Models}

Some notions in this subsection are adapted from \cite{CGZ}.

\begin{definition}\label{def:model} A language model $M$ over token set $\mathcal{T}$ is a deterministic algorithm that takes as input a prompt $\mathsf{prompt}$ and tokens previously output by the model $z = (z_1, \ldots, z_{i-1})$, and outputs a probability distribution $p_i = M(\mathsf{prompt}, z)$ over $\mathcal{T}$.
\end{definition}

\begin{definition} A language model's response to $\mathsf{prompt}$ is a random variable $M(\mathsf{prompt}) \in \mathcal{T}^*$ defined algorithmically as follows. We begin with an empty list of tokens $z = ()$. As long as the last token in $z$ is not $\mathsf{done}$, we draw a token $z_i$ from the distribution $M(\mathsf{prompt}, z)$ and append it to $z$.
\end{definition}

\subsection{Entropy and Empirical Entropy}

For a distribution $D$ over a finite set $X$, the Shannon entropy is
\[
H(D) = \mathbb{E}_{z \sim D}[-\log D(z)].
\]
The empirical entropy of $z$ in $D$ is $-\log D(z)$.

\begin{definition}
For a language model $M$, a prompt $\mathsf{prompt}$, and a response $z \in \mathcal{T}^*$, the empirical entropy of $M$ responding with $z$ to $\mathsf{prompt}$ is:
\[
H_e(M, \mathsf{prompt}, z) := -\log \Prob[M(\mathsf{prompt}) = z].
\]
\end{definition}

Note that
\[
H(M(\mathsf{prompt})) = \mathbb{E}_z[H_e(M,\mathsf{prompt},z)].
\]


\subsection{Time-Lock Puzzles}

Time-lock puzzles, introduced by \citet{rivest1996time}, provide a mechanism for sending messages ``to the future''. The core idea is to enable a sender to quickly generate a cryptographic puzzle whose solution is guaranteed to remain hidden until a specified amount of computational time has elapsed. 
Additionally, this security must hold even against highly parallel adversaries equipped with polynomially many processors. 
The following is a definition of a Time-Lock Puzzle adapted from \cite{shafitimelocks}.

\begin{definition}[Time-Lock Puzzle]\label{def:tlp}
A \emph{Time-Lock Puzzle} (TLP) is defined over a difficulty parameter $t \in \mathbb{N}$ and a security parameter $\secp$. 
A Time-Lock Puzzle consists of a pair of algorithms $\mathsf{Puzzle} = (\mathsf{Gen}, \mathsf{Sol})$ with the following properties:

\begin{itemize}
    \item \textbf{Puzzle Generation:}
    \[
    Z \gets \mathsf{Gen}(t, s)
    \]
    is a probabilistic algorithm that takes as input a difficulty parameter $t$ and a solution $s \in \{0,1\}^\lambda$ and outputs a puzzle $Z$.
    \item \textbf{Puzzle Solving:}
    \[
    s \gets \mathsf{Sol}(Z)
    \]
    is a deterministic algorithm that takes a puzzle $Z$ and returns a solution $s$.

    \item \textbf{Completeness:} For all difficulty parameters $t \in \mathbb{N}$ and $s \in \{0,1\}^\secp$:
    \[
    \mathsf{Sol}(\mathsf{Gen}(t,s)) = s.
    \]

    \item \textbf{Efficiency:}
    \begin{itemize}
        \item $\mathsf{Gen}(t, s)$ runs in time $\mathsf{poly}(\log t, \lambda)$.
        \item $\mathsf{Sol}(Z)$ runs in time $t \cdot \mathsf{poly}(\lambda)$.
    \end{itemize}

    \item \textbf{Sequentiality (Time-Lock Property):} 
    There exists $\varepsilon < 1$ and a polynomial $\underline{t}(\cdot)$, such that for every polynomial $t(\cdot) \geq \underline{t}(\cdot)$ and every polysize adversary $\mathcal{A} = \{\mathcal{A}_\lambda\}_{\lambda \in \mathbb{N}}$ of depth $\mathrm{dep}(\mathcal{A}_\lambda) \leq t^{\varepsilon}(\lambda)$, there exists a negligible function $\mu$, such that for every $\lambda \in \N$, and every pair of solutions $s_0, s_1 \in \{0,1\}^\lambda$:
        \[
        \Pr\Big[
        b \leftarrow \mathcal{A}_\lambda(Z) : b \leftarrow \{0,1\},
        Z \leftarrow \mathsf{Gen}(t(\lambda), s_b)
        \Big]
        \leq \frac{1}{2} + \mu(\lambda).
        \]
\end{itemize}
\end{definition}

\subsubsection{The RSW Construction}

\citet{rivest1996time} proposed a candidate construction based on the conjectured sequential nature of modular exponentiation. Their scheme uses the fact that repeated squaring modulo an RSA integer is believed to be inherently sequential.

$\mathsf{Gen}(t)$ proceeds as follows:

\begin{enumerate}
    \item Samples two primes $p, q$ from $[1,w]$, for a parameter $w$, and set $n = pq$.
    \item Samples $r \gets \mathbb{Z}_n^*$, where $\Z_n^* = \{1,2,\dots, n-1\}$, and computes
    \[
        s \gets r^e \bmod n,
    \]
    where $e = 2^{t}$ is a tower of repeated squarings (i.e., $t$ squarings of $r$).
    To compute $r^e \bmod n$ it first evaluates $\phi(n) = (p-1)(q-1)$ and reduces the exponent $e' \gets e \bmod \phi(n)$.
    Then $s = r^{e'} \bmod n$, which can be computed in only $O(\log(e'))$ modular exponentiations.
    \item Defines $Z = (n, r)$ to be the puzzle and $s$ its solution. Returns $(Z,s)$.
\end{enumerate}

\begin{note}
$\mathsf{Gen}$ can be adapted to accept an additional input $s$ so that it adheres to the schema from Definition~\ref{def:tlp}. 
To do that, one would use a $\mathsf{root}$ function (Definition~\ref{def:root}) that for $x \in \Z_n^*$ computes \say{$\sqrt{x}$} so that it is possible to compute $r$ from $s$. However, there are some technical details that need to be addressed as $\mathsf{root}$ is well defined only on a subgroup of $\Z_n^*$ (see Fact~\ref{fact:residues}). We implicitly address these issues in Lemma~\ref{lem:hardcore}. We chose to use this formulation of the RSW construction for simplicity of exposition.    
\end{note}

$\mathsf{Sol}(Z)$ works as follows:
\begin{enumerate}
    \item Parses $Z = (n,r)$.
    \item Returns $r^{2^t}$.
\end{enumerate}

Note that $\mathsf{Sol}$ computes $s$ using $t$ repeated squarings, each of which is inexpensive, but the full process requires $\Theta(t)$ sequential steps.

\paragraph{Sequentiality.} The security of the RSW puzzle hinges on the assumption that modular exponentiation---specifically, computing $r^{2^t} \bmod n$---cannot be substantially accelerated via parallelism. That is any adversary who does not know the factorization of $n$ must essentially compute each squaring step sequentially.

This assumption is supported by decades of experience in cryptographic implementation and is believed to be secure. 
Although no formal proof of sequentiality is known, the construction remains a compelling candidate and has served as the basis for more advanced cryptographic primitives, such as time-release encryption \citep{bonehVDF, bonehtimedcommitments}.

\subsection{Public key encryption}
\citet{diffie1976} introduced the notion of public key encryption (PKE), enabling secure communication between parties without a shared secret. In a PKE scheme, each user generates a pair of keys: a public key $\mathsf{pk}$ that is distributed openly and a private key $\mathsf{sk}$ that is kept secret. Anyone can encrypt a message $m$ using $\mathsf{pk}$ to obtain a ciphertext $c$, but only the holder of $\mathsf{sk}$ can decrypt $c$ to recover $m$. Security relies on the assumed hardness of inverting the encryption without access to the secret key---for example, the difficulty of factoring large integers in the RSA scheme.

\subsection{Secret key encryption}
Secret key encryption (also known as symmetric-key encryption) requires both parties to share a common secret key $\mathsf{k}$ in advance. The encryption algorithm uses $\mathsf{k}$ to map a message $m$ to a ciphertext $c$, and the decryption algorithm uses the same key $\mathsf{k}$ to recover $m$ from $c$. Symmetric schemes are typically much more efficient than public key schemes and form the backbone of practical secure communication when a secure channel for key exchange is already available. Notable constructions include block ciphers such as the Advanced Encryption Standard (AES) \citep{aes} and stream ciphers built from pseudorandom generators.

\section{Recoverable-Randomness Sampling}

\begin{definition}
For $\alpha : \N \rightarrow \N$ we say that a pair of algorithms $(\mathsf{Samp},\mathsf{RecoverRand})$ is an $\alpha$-Recoverable-Randomness Sampling (RRS) scheme if
\begin{itemize}
    \item $\mathsf{Samp}(G) \rightarrow y$, given a language model $G$\footnote{See Remark~\ref{rem:G} for why $G$ is a language model.} over token set $\mathcal{T}$, $\mathsf{Samp}$ samples an output $z \in \mathcal{T}^*$.
    \item $\mathsf{RecoverRand}(G,z,L) \rightarrow R$, for $y \in \mathcal{T}^*$, a language model $G$, $\mathsf{RecoverRand}$, and $L \in \N$ is a deterministic algorithm that recovers $L$ bits of the randomness used by $\mathsf{Samp}(G)$ when generating $z$.
\end{itemize}
The pair $(\mathsf{Samp},\mathsf{RecoverRand})$ satisfies
\begin{itemize}
    \item{\textbf{Recoverability.}} For every $R \in \{0,1\}^*$ and every $L \in \N, L \leq \mathsf{len}(R)$ if $\mathsf{Samp}(G,R)$ halts and returns $z_R = \mathsf{Samp}(G,R)$ then
    \[
    \mathsf{RecoverRand}(G,z_R,L) = R[0:L],
    \]
    In the notation $\mathsf{Samp}(G,R)$, $R$ explicitly specifies the randomness used.
    \item{\textbf{Closeness of Distributions.}}
    Distirbutions over $\mathcal{T}^*$ given by $\mathsf{Samp}(G)$ and $G$ are statistically close.
    Formally, for every $L \in \N$, 
    \[
    \frac12 \sum_{z \in \mathcal{T}^*, |z| \leq L} \Big| \Prob_R[\mathsf{Samp}(G,R) = z] - \Prob[G = z] \Big| \leq \alpha(L).
    \]
    Note that the expression on the left-hand side evaluates a partial statistical distance.
\end{itemize}
\end{definition}

\begin{remark}\label{rem:G}
In Theorems~\ref{thm:detection-impossible-timelock},~\ref{thm:detection-impossible-sk} and~\ref{thm:detection-impossible-pk} model $G$ is a generator of innocent prompts.
The theorems hold for \emph{any} high-entropy model.    

We assume that $M$ (see for instance Theorem~\ref{thm:detection-impossible-timelock}) is deployed on some distribution $\dist$, e.g., a distribution of a user prompts for a chatbot.
We think of $G$ as a model generating prompts from $\dist$ or approximation thereof.
It is natural to assume such a $G$ exists.
In the example above, the top chatbots are used by millions of users, so finding a single $G$ whose outputs are indistinguishable from some users' prompts should be easy.
\end{remark}

\subsection{RRS Construction}

We first simplify the definition of a language model (Definition~\ref{def:model}) by assuming that the token set is binary, $\mathcal{T} = \{0,1\}$. 
We may assume this without loss of generality due to a straightforward reduction that appears in Section 4.1 of \cite{CGZ}. We will implicitly use this reduction throughout our work as well.
Our construction is given in Algorithms~\ref{alg:samp} and~\ref{alg:rec}.

\begin{remark}
Our construction was inspired by \cite{zamir2024undetectable}, who design a method to undetectably embed a secret message in the output of an LLM that can later be recovered provided access to the secret key. 
This is similar to the requirements of the RRS, where the \textbf{randomness} needs to be recoverable from the output.
\end{remark}

There are some details in the construction of RRS that we didn't address in Section~\ref{sec:overview}.
The first is how do we set the precision parameter $P$.
As we will see shortly in Theorem~\ref{thm:recoveringworks}, $P$ influences the closeness of the distribution generated by $\mathsf{Samp}$ to that of $G$.
$P$ will be set to an exact value when $\mathsf{Samp}$ will be used as a subroutine in the proof of a formal version of Theorem~\ref{thm:detection-impossible-timelock}.
It is also informative to point out why the two distributions can be different in the first place.
The main reason is that $\mathsf{Samp}$ samples $z_i = 1$ with probability $\approx p_i$ (and not $= p_i$).
It is because we didn't assume anything about the precision of $p_i$ and it might \say{fall} in between the grid points defined by $q/k$, where $q \in [k]$.




\subsection{Proof of correctness}

\begin{theorem}
\label{thm:recoveringworks}
Algorithms~\ref{alg:samp} and~\ref{alg:rec} constitute an $O(L \cdot 2^L \cdot 2^{-P})$-Recoverable-Randomness Sampling scheme.
\end{theorem}

\begin{proof}
The result follows from a series of facts.

$\mathsf{Samp}$ satisfies the following properties:
\begin{enumerate}
    \item $q \in [k]$ at all times.
    \item for every iteration $i$:
    \begin{itemize}
        \item for every $z_0,\dots,z_{i-1}$, before the if statement (step 8 of Algorithm~\ref{alg:samp}) is executed, the distribution of $q$ is uniform in $[k]$.
    \end{itemize}
\end{enumerate}
Observe that if these properties hold, then $z_i$'s are sampled from a distribution that is close to $G$.
The only reason the distributions might be different is the aforementioned precision issue.
However, step 4 of Algorithm~\ref{alg:samp} ensures that $|q/k-p_i| \leq 2^{-P}$, so the errors accumulate as required, as for every $(z_0,\dots,z_{L-1})$ the difference between sampling probabilities is upper bounded by $O(L \cdot 2^{-P})$.

The above properties are proven by induction over $i$.
Because $R$ is a random bit-string it implies that $q$ is uniform in $[0,\ceil{p_i \cdot k}]$ conditioned on the event $q/k < p_i$.
This implies that if we update $k \gets \ceil{p_i \cdot k}$ then the property that $q$ is uniform in the updated $[k]$ is satisfied.

The properties of the $\mathsf{RecoverRand}$ algorithm are as follows.
For every iteration $i$: 
\begin{enumerate}
    \item $k$ and $d$ at the beginning of the iteration are equal to $k$ and $d$ in the $\mathsf{Samp}$ algorithm at the beginning of iteration $i$.
    \item the lengths of $a,b$ and $q$ from $\mathsf{Samp}$ are equal before the if statement (step 8 of Algorithm~\ref{alg:rec}).
    \item $a \leq_\text{lex} q \leq_\text{lex} b$, where $\leq_\text{lex}$ is the lexicographic order, before the if statement (step 8 of Algorithm~\ref{alg:rec}).
\end{enumerate}
Similarly, as for $\mathsf{Samp}$ these properties can be proven by induction over $i$.

The first property follows from the observation that $k$ and $d$ are updated depending only on the $z_i$ values.
The second property follows from the first property.
The third property is the crucial one and follows from the previous two and a careful analysis of how $a,b$ are updated in the if statement.
More concretely, note that $a + (b-a)\cdot p_i/k$ is the \say{threshold randomness} at which $\mathsf{Samp}$ would generate $z_i = 0$ vs $z_i = 1$.
Indeed, if $R \in [a,a + (b-a)\cdot p_i/k)$ then conditioned on $\mathsf{Samp}$ generating $z_0,\dots,z_{i-1}$, $\mathsf{Samp}$ would set $z_i \gets 1$.
If $R \in [a + (b-a)\cdot p_i/k),b)$ then it would generate $z_i = 0$.
\end{proof}

\section{Impossibility of Prompt-Detection Filters}\label{apx:promptdetection}

\subsection{Moderately-strong filters}\label{apx:timelock}

We first state a formal version of Theorem~\ref{thm:detection-impossible-timelock}.
Next, we will present the construction and prove its correctness in the remaining part of this section.

\begin{theorem}[Impossibility of prompt-detection filters]
\label{thm:detection-impossible-timelock-formal}
There exists $\e \in (0,1)$, such that for all security parameters $\secp \in \N$, every time parameter $\poly(\secp) \leq t \leq 2^{O(\secp)}$, every set of harmful prompts $S_H \subseteq \{0,1\}^{l_m}$, where $l_m = \Theta(\log(\secp)), \Omega(\log(\secp)) = \log(|S_H|) \leq l_m - \Omega(\log(\secp))$, for every language model $M$ running in time $O(t)$ per token, for every 
prompt generator $G$ satisfying $H(G) \geq \poly(\secp)$ and running in time $\poly(\secp)$ per token the following conditions hold.

There exists an adversarial prompt generator $G'$ with a running time $\poly(\secp)$ per token and a language model $M'$ with a running time $\poly(\secp) \cdot O(t)$ per token such that the following conditions hold.
\begin{enumerate}
    \item{\textbf{($G'$ incites harmful behavior in $M'$)}.} For every malicious prompt $m \in S_H$, the following two distributions are close:
    \[
    \left\{M'(z) \ \middle| \ z \gets G'(m) \right\}, \left\{M(m) \right\}.
    \]
    In words, for every harmful prompt from $S_H$ the outputs of $M'$ and $M$ follow almost the same distribution, in particular, a similar harmful behavior is induced for $M$ and $M'$.
    \item{\textbf{($M'$ and $M$ behave identically on $G$ almost always)}.} With probability $1 - \secp^{-\Omega(1)}$ over $z \gets G$ the distributions $M'(z)$ and $M(z)$ are equal.
    \item{\textbf{($G$ and $G'$ are indistinguishable for filters running in $o(t^\eps)$)}.} Under the assumption that RSW is secure, for every prompt filter running in time $o(t^\eps)$, for every $m \in S_H$ the advantage in distinguishing outputs of $G$ and $G'(m)$ is $\secp^{-\Omega(1)}$.
    In words, no filter running in time $o(t^\eps)$ can reject\footnote{If it doesn't also reject the innocent prompts.} the adversarial prompts from $G'$.
\end{enumerate}

\end{theorem}

\subsubsection{Construction}

In this subsection, we give a more detailed version of the construction provided in Section~\ref{sec:overview}.
Our construction is given in Algorithms~\ref{alg:harminciteformal} and~\ref{alg:maliciousllmformal}, which are formal versions of Algorithms~\ref{alg:harmincite} and~\ref{alg:maliciousllm}.

\begin{algorithm}
\caption{\(\mathsf{HarmIncite}(G,m,t,\secp)\) \Comment{$G$: an innocent prompt generator; $m$: a harmful prompt; $t$: time parameter; $\secp$: security parameter}}\label{alg:harminciteformal}
\begin{algorithmic}[1]
\State Sample two primes \(p,q\) in $[1,w]$ \Comment{$w = 2^{O(\secp)}$}
\State Set $n \gets p \cdot q$
\State Compute $e \gets 2^t \pmod{(p-1)(q-1)}$ \Comment{$\phi(n) = (p-1)(q-1)$}
\State Sample $r \gets \Z_n$ 
\If{$r \in \Z_n^*$} \Comment{See Remark~\ref{rem:differentsampling}}
    \State Compute $(g,i) \gets \mathsf{ind}(r)$ \Comment{See Definition~\ref{def:bijection}}
    \State Compute $g' \gets g^e \pmod{n}$
    \State Set $s \gets \mathsf{ind}^{-1}(g', i)$
\Else
    \State Set $s \gets r$
\EndIf
\State Sample a pairwise independent hash function \(h\) 
\State Sample $v \gets \{0,1\}^{O(\secp)}$
\State $c \gets Q(s,v) \oplus m$ \Comment{$Q$ is the hard-core function}
\State \(R \gets h \ \| \ h(n) \ \| \ r \ \| \ v \ \| \ c\) \Comment{Concatenate components}
\State \Return \(z \gets \mathsf{Samp}(G, R)\) \Comment{Sample output so that randomness is recoverable}
\end{algorithmic}
\end{algorithm}

\begin{algorithm}
\caption{\(\mathsf{MaliciousLLM}(M, z, G, t, \secp, S_H)\) \Comment{\(M\): an LLM; \(z\): input; \(G\): weak innocent prompt generator; $t$: time parameter; $\secp$: security parameter; $S_H$: set oh harmful prompts}}\label{alg:maliciousllmformal}
\begin{algorithmic}[1]
\State \(R \gets \mathsf{RecoverRand}(G, z)\)
\State Parse \(R\) as \(h \ \| \ u \ \| \ r \ \| \ v \ \| \ c\)
\ForAll{\(\hat{n} \in h^{-1}(u)\)}
    \If{$r \in \Z_n^*$} \Comment{See Remark~\ref{rem:differentsampling}}
    \State Compute $(g,j) \gets \mathsf{ind}(r)$ \Comment{See Definition~\ref{def:bijection}}
    \State Compute $\hat{g} \gets g^e \pmod{\hat{n}}$
    \State Set $\hat{s} \gets \mathsf{ind}^{-1}(\hat{g}, j)$
    \Else
        \State Set $\hat{s} \gets r$
    \EndIf
    \For{\(i \in [t]\)}
        \State \(\hat{s} \gets \hat{s}^2 \text{ mod } \hat{n}\)
    \EndFor
    \State \(m \gets \hat{s} \oplus c\) \Comment{$m$ is the malicious prompt extracted from ciphertext $c$}
    \If{\(m \in S_H\)}
        \State \Return \(M(m)\) \label{line:returnharmful}
    \EndIf
\EndFor
\State \Return \(M(z)\) \label{line:returnnormal}
\end{algorithmic}
\end{algorithm}

Next, we provide a few simple facts from number theory that are needed to address the differences between the formal and informal versions of the algorithms.

\begin{definition}
Let $n \in N$.
We call $x \in \Z_n^*$ a quadratic residue modulo $n$ if there exists $r$ such that $x = r^2 \bmod n$.
\end{definition}

The following is a standard fact from number theory. See \cite{odedfoundations} for details.

\begin{fact}\label{fact:residues}
Let $n \in \N$ be a product of two different odd primes, i.e., $n = pq$.
The set of quadratic residues modulo $n$ is a subgroup, denoted by $G_n$, of $\Z_n^*$ of size $|\Z_n^*|/4$.
Moreover, the mapping $x \mapsto x^2$ is a 2-to-1 mapping from $\Z_n^*$ to $G_n$.
\end{fact}

\begin{definition}\label{def:bijection}
For $n \in \N$ we define a bijection $\mathsf{ind} : \Z_n^* \rightarrow G_n \times \{0,1\}^2$ as $\mathsf{ind}(x) := (g,i)$, where $g \in G_n$ is such that $x^2 = g^2$ and $i$ denotes which (in the increasing order) of the $4$ elements $r \in \Z_n^*$ satisfying $r^2 = x^2$, $x$ is.   
\end{definition}

\begin{remark}\label{rem:differentsampling}
Note that in Algorithm~\ref{alg:harminciteformal} (in contrast to Algorithm~\ref{alg:harmincite}), to compute the puzzle $s$, an intermediate bijection $\mathsf{ind}$ (Definition~\ref{def:bijection}) is used.
The reason is the following.
The mapping $x \mapsto x^2$ is a $4$-to-$1$ function in $\Z_n^*$.
To ensure that the mapping between $r \in \Z_n$ and a puzzle $s$ is $1$-to-$1$, we use the function $\mathsf{ind}$ to \say{remember} which of the $4$ possible $r$'s was used.
\end{remark}

Next, we give a formal definition of $\mathsf{Enc}$ used in Algorithm~\ref{alg:harmincite}.

\paragraph{Hardcore functions.}
We will apply the technique of hardcore bits (\cite{GoldreichL89}) and its extension to hardcore functions. 
Recall that a hardcore bit is a function $B$ that outputs a single bit $B(x)$ which is hard to predict given $f(x)$. 
A hardcore function $Q$ for a one-way function $f$ is a function which outputs possibly more than one bit such that, given $f(x)$, and the description of $Q$, it is hard to distinguish $Q(x)$ from a random string even when $f(x)$ is known.  
Goldreich and Levin introduced the first hardcore predicates and functions for general one-way functions (OWFs), showing that a random linear function is hardcore and the linear function defined by a random Toeplitz matrix is a hardcore function.

The first step of $G'$ is to sample $p,q$ at random from a set $[1,k]$ (see Algorithm~\ref{alg:harminciteformal}), where $k = 2^{O(\secp)}$.
Next, we let $h$ be a pairwise independent hash function that $h$ hashes $n$ to a range of size $\Theta(\e^2 \cdot ((k/\log(k))^2)$.
The $\e$ is set to $\secp^{-\gamma}$ for some $\gamma > 0$ to be fixed in the proof (Section~\ref{sec:mainformalproof}).
For a message $m \in \{0,1\}^{l_m}$ (as in the statement of Theorem~\ref{thm:detection-impossible-timelock-formal}) and using the notation from Algorithm~\ref{alg:harminciteformal} we define $R$ to be
\begin{equation}\label{eq:hidemessage}
R \gets h \ \| \ h(n) \ \| \ r \ \| \ v \ \| \ Q(s,v) \oplus m,
\end{equation}
where $v \gets \{0,1\}^{O(\secp)}$ and $Q(s,v)$ is a hardcore function with $l_m = \Theta(\log(\secp))$ output bits.

\begin{note}
In Section~\ref{sec:overview} we simplified the above construction and claimed $R = (x,\mathsf{Enc}(g(x),m))$.
Informally speaking, in the formal version we set 
\[
x = h \ \| \ h(n) \ \| \ r, \ \mathsf{Enc}(g(x),m) = v \ \| \ Q(s,v) \oplus m,
\]
where $v \gets \{0,1\}^{O(\secp)}$.
\end{note}

Next, we give some remarks about Algorithm~\ref{alg:maliciousllmformal}.
As we mentioned in the overview it is possible to efficiently enumerate through all the preimages of $h(n)$ under $h$.
Note that according to \eqref{eq:hidemessage} when $M'$ parses $R$ it does it as follows:
\[
R = h \ \| \ u \ \| \ r \ \| \ v \ \| \ c.
\]
For instance, an additional part $v$, which was not present in the simplified version of Algorithm~\ref{alg:rec} appears here.

Next, the model $M'$ squares the starting point $t$ times and tries to decrypt $c$ with the current candidate for $n$, i.e., $\hat{n}$.
By decrypt we mean it evaluates $\mathsf{Dec}(\hat{s},v,c)$ defined as 
\[
\hat{m} \gets c \oplus Q(\hat{s},v).
\]
The check $\hat{m} \neq \bot$ becomes instead $m \in S_H$. 
That is we assume $M'$ knows the set of harmful prompts $S_H$ and checks if the decryption belongs to that set and if yes it executes $M$ on it (which will lead to a harmful output).
If, none of the the decryptions \say{succeed} then it returns $M(z)$.


\begin{remark}\label{rem:exttoTLP}
We proved Theorem~\ref{thm:detection-impossible-timelock-formal} using the RWS construction of time-lock puzzles.
However, there are generic constructions of time-lock puzzles.
The first one \citep{shafitimelocks} was based on obfuscation and randomized encodings.
A culmination of a series of follow-up works \citep{giulioTLP, bitanskyTLP, giulioTLPfinal} managed to construct time-lock puzzles based on a version of the Learning With Errors \citep{regevLWE} (LWE) assumption. 
One advantage of basing TLPs on LWE is that the other constructions are broken by quantum algorithms, while LWE is a candidate for post-quantum secure cryptography.

It would be interesting to extend Theorem~\ref{thm:detection-impossible-timelock-formal}, so that it can be based on any time-lock puzzle.
The one crucial property that we use is the pseudorandomness of the puzzle, and this seems like the key issue in generalizing the result. 
Recall that we had to be careful in designing Algorithm~\ref{alg:harminciteformal} so that the distribution of $s$ is uniform.
It seems that the main property we need is that the distribution on hard instances is uniform over a set $S$ whose approximate size is known to us.
We leave the question of whether the generalization is possible for future work.
\end{remark}

\subsubsection{Proof of Theorem~\ref{thm:detection-impossible-timelock-formal}}\label{sec:mainformalproof}

Firstly, we prove two technical lemmas.

\begin{lemma}\label{lem:hnrandom}
$h \ \| \ h(n)$ is $O(\secp^{-2\gamma})$-close to uniformly random.    
\end{lemma}

\begin{proof}
It follows from the Leftover Hash Lemma (\cite{Vadhan12}, Theorem 6.18), which guarantees that $h \| h(n)$ is $\eps^2$ close to uniform.
For our setting of parameters it gives us that $h \| h(n) $ is $O(\secp^{-2\gamma})$-close to uniform.   
\end{proof}

\begin{definition}\label{def:root}
For $n \in \N$ let $\mathsf{root} : G_n \rightarrow G_n$ be a function that on input $x \in G_n$ returns $r$ such that $r^2 = x$ (by Fact~\ref{fact:residues} only one exists).   
\end{definition}

\begin{lemma}\label{lem:hardcore}
There exists an $\Theta(\log(\secp))$-bit hardcore function $Q(s,v)$ (see \cite{odedfoundations} for details) for a function $f : \Z_n \rightarrow \Z_n$ defined as follows.
Let $s \in \Z_n$ and consider cases:
\begin{itemize}
    \item if $s \in \Z_n \setminus \Z_n^*$ then $f(s) = s$.
    \item if $s \in \Z_n^*$ then 
    \[
    f(s) := \mathsf{ind}^{-1}(\mathsf{root}^t(\mathsf{ind}(s)_1),\mathsf{ind}(s)_2).
    \]
\end{itemize}    

More formally, for every distinguisher running in time $o(t^\e)$, the advantage for distinguishing
\[
r \ \| \ v \ \| \ Q(s,v) \text{ and } r \ \| \ v \ \| \ U,
\]
where $s \gets \Z_n$, $U \gets \{0,1\}^{l_m}, v \gets \{0,1\}^{O(\secp)}$ and $r = f(s)$, is at most $\secp^{-\Omega(1)}$.

Moreover, $f$ is a bijection.
\end{lemma}

\begin{proof}
The proof is a direct adaptation of the proof of Theorem 2.5.6 in \cite{odedfoundations}.
The main difference is that security in our case holds only against distinguishes running in time $o(t^\eps)$ and not all polynomial-time adversaries.

The fact that $f$ is a bijection follows from Fact~\ref{fact:residues} and the fact that $\mathsf{ind}$ (Definition~\ref{def:bijection}) is a bijection.
\end{proof}

Note that the distribution considered in Lemma~\ref{lem:hardcore} is, a priori, different from the distribution of outputs of Algorithm~\ref{alg:harminciteformal}, because in the actual algorithm we first sample $r$, not $s$.
The following lemma shows that the distributions are in fact equal.

\begin{lemma}\label{lem:equaldistribution}
The following two distributions are equal. 
The first distributions is defined as $r \ \| \ v \ \| \ Q(s,v)$,
where $s \gets \Z_n$, $U \gets \{0,1\}^{l_m}, v \gets \{0,1\}^{O(\secp)}$ and $r = f(s)$ (where $f$ is defined in Lemma~\ref{lem:hardcore}).
The second distribution is defined as $r \ \| \ v \ \| \ Q(s,v)$,
where $r \gets \Z_n$, $U \gets \{0,1\}^{l_m}, v \gets \{0,1\}^{O(\secp)}$ and $s = f^{-1}(r)$.
\end{lemma}

\begin{proof}
The equality follows from Lemma~\ref{lem:hardcore}, which gurantees that $f$ is a bijection.    
\end{proof}

Next, we are ready to prove Theorem~\ref{thm:detection-impossible-timelock-formal}.

\begin{proof}
There are three conditions we need to verify.
\begin{itemize}
    \item{Condition 1.} From the assumption that $H(G) \geq \poly(\secp)$ we know that the length of randomness needed to sample from $G$ is at least the length needed to describe $h \ \| \ h(n) \ \| \ r \ \| \ v \ \| \ c$ with all but negligible probability.
    To do that it is enough to set the precision parameter $P = \poly(\secp)$ as guaranteed by Theorem~\ref{thm:recoveringworks}.
    Conditioned on the fact that $\mathsf{RecoverRand}$ recovers enough bits, $M'$ will recover $m$ with high probability and return $M(m)$.
    
    \item{Condition 2.} We bound the probability that when $M'$ (defined via Algorithm~\ref{alg:maliciousllmformal}) receives a prompt generated by $G$ it returns in Line~\ref{line:returnharmful}, i.e., it returns a harmful output.
    Note that $c$ is uniformly random as $G$ uses honest randomness $v$ to produce its outputs.
    Similarly, $s$ is also uniformly random.
    This implies that, for every $\hat{n} \in h^{-1}(u)$, the probability that $m \in S_H$ is equal to $|S_H|/2^{l_m}$.
    From the setting $\eps = \secp^{-\gamma}$ we know that the number of preimages under $h$ is in $O(\secp^{2\gamma})$\footnote{Thus, $M'$ can iterate over all of them in time $O(t)$.}
    and so, by the union bound, the probability that at least one of $m$'s belongs to $S_H$ is at most $O(\secp^{2\gamma} \cdot |S_H|/2^{l_m})$.
    From the assumption that $\log(|S_H|) \leq l_m - \Omega(\log(\secp))$ it follows that $\gamma$ can be set so that the probability that $M'$ returns in Line~\ref{line:returnharmful} is at most $\secp^{-\Omega(1)}$.
    By construction, if $M'$ returns in the Line~\ref{line:returnnormal}, it returns $M(z)$, thus returning from the same distribution.

    \item{Condition 3.} Implied by Lemmas~\ref{lem:hnrandom},~\ref{lem:hardcore}, and~\ref{lem:equaldistribution} and the fact that the probability that $r \in \Z_n^*$ (note that it is sampled from $\Z_n$) is negligible in $\secp$.
\end{itemize}

\end{proof}

\subsection{Shared secret key}

The following is a formal version of Theorem~\ref{thm:detection-impossible-sk}.

\begin{theorem}[Impossibility of prompt-detection filters via secret keys]
\label{thm:detection-impossible-secret key-formal}
There exists a polynomial $p$ such that for all security parameters $\secp \in \N$, every set of harmful prompts $S_H \subseteq \{0,1\}^{l_m}$, where $l_m = \Theta(p(\secp)), \Omega(p(\secp)) = \log(|S_H|) \leq l_m - \Omega(p(\secp))$, for every language model $M$, for every 
prompt generator $G$ satisfying $H(G) \geq \poly(\secp)$ and running in time $\poly(\secp)$ per token the following conditions hold.

There exists an adversarial prompt generator $G'$ with a running time $\poly(\secp)$ per token and a language model $M'$ with a running time $\poly(\secp) + O(T_M)$ per token, where $T_M$ is the running time of $M$ per token, such that if $M'$ and $G'$ share an $\poly(\secp)$-bit secret key then the following conditions hold.
\begin{enumerate}
    \item{\textbf{($G'$ incites harmful behavior in $M'$)}.} For every malicious prompt $m \in S_H$, the following two distributions are close:
    \[
    \left\{M'(z) \ \middle| \ z \gets G'(m) \right\}, \left\{M(m) \right\}.
    \]
    In words, for every harmful prompt from $S_H$ the outputs of $M'$ and $M$ follow almost the same distribution, in particular, a similar harmful behavior is induced for $M$ and $M'$.
    \item{\textbf{($M'$ and $M$ behave identically on $G$ almost always)}.} With all but negligible in $\secp$ probability over $z \gets G$ the distributions $M'(z)$ and $M(z)$ are equal.
    \item{\textbf{($G$ and $G'$ are indistinguishable for filters running in polynomial-time)}.} Under the assumption that One-Way Functions exist, for every prompt filter running in polynomial-time, for every $m \in S_H$ the advantage of the filter in distinguishing outputs of $G$ and $G'(m)$ is negligible in $\secp$.
    In words, no filter running in polynomial-time can reject\footnote{If it doesn't also reject the innocent prompts.} the adversarial prompts from $G'$.
\end{enumerate}
\end{theorem}

\subsubsection{Construction}

As we discussed in Section~\ref{sec:overview}, instead of using time-locks, we use symmetric key (shared secret key) encryption (it is known that one-way functions (OWFs) imply symmetric key encryption).
Namely, for $m \in S_H$ we set $R := \mathsf{Enc}(\mathsf{sk},m)$, where $\mathsf{Enc}$ is the symmetric key encryption scheme, and $\mathsf{sk}$ is the shared secret key.
Next, $G'$ generates $z \gets \mathsf{Samp}(G,R)$ as before.
$M'$ decrypts the randomness $R$ to get $\hat{m}$ and checks if $\hat{m}$ belongs to the set $S_H$, and if yes returns $M(\hat{m})$. 

The proof of correctness is a simplified version of the proof of Theorem~\ref{thm:detection-impossible-timelock-formal}.

\subsection{Shared public key}

The following is a formal version of Theorem~\ref{thm:detection-impossible-pk}.

\begin{theorem}[Impossibility of prompt-detection filters via secret keys]
\label{thm:detection-impossible-public key-formal}
There exists a polynomial $p$ such that for all security parameters $\secp \in \N$, every set of harmful prompts $S_H \subseteq \{0,1\}^{l_m}$, where $l_m = \Theta(p(\secp)), \Omega(p(\secp)) = \log(|S_H|) \leq l_m - \Omega(p(\secp))$, for every language model $M$, for every 
prompt generator $G$ satisfying $H(G) \geq \poly(\secp)$ and running in time $\poly(\secp)$ per token the following conditions hold.

There exists a language model $M'$ with an associated public key and with a running time $\poly(\secp) + O(T_M)$ per token, where $T_M$ is the running time of $M$ per token, and an adversarial prompt generator $G'$ (that depends on the public key) with a running time $\poly(\secp)$ per token and such that the following conditions hold.
\begin{enumerate}
    \item{\textbf{($G'$ incites harmful behavior in $M'$)}.} For every malicious prompt $m \in S_H$, the following two distributions are close:
    \[
    \left\{M'(z) \ \middle| \ z \gets G'(m) \right\}, \left\{M(m) \right\}.
    \]
    In words, for every harmful prompt from $S_H$ the outputs of $M'$ and $M$ follow almost the same distribution, in particular, a similar harmful behavior is induced for $M$ and $M'$.
    \item{\textbf{($M'$ and $M$ behave identically on $G$ almost always)}.} With all but negligible in $\secp$ probability over $z \gets G$ the distributions $M'(z)$ and $M(z)$ are equal.
    \item{\textbf{($G$ and $G'$ are indistinguishable for filters running in polynomial-time)}.} Under the assumption that public key cryptography exists, for every prompt filter running in polynomial-time, for every $m \in S_H$ the advantage of the filter in distinguishing outputs of $G$ and $G'(m)$ is negligible in $\secp$.
    In words, no filter running in polynomial-time can reject\footnote{If it doesn't also reject the innocent prompts.} the adversarial prompts from $G'$.
\end{enumerate}
\end{theorem}

\subsubsection{Construction}

Instead of using time-locks (as in Theorem~\ref{thm:detection-impossible-timelock-formal}) we use public key encryption in our construction.
Namely, for $m \in S_H$ we set $R := \mathsf{Enc}(\mathsf{pk},m)$, where $\mathsf{Enc}$ is the public key encryption scheme, and $\mathsf{pk}$ is the public key associated with $M$.
Next, $G'$ generates $z \gets \mathsf{Samp}(G,R)$ as before.
$M'$ decrypts the randomness $R$ to get $\hat{m}$ and checks if $\hat{m}$ belongs to the set $S_H$, and if yes returns $M(\hat{m})$. 

The proof of correctness of this construction is a simplified version of the proof of Theorem~\ref{thm:detection-impossible-timelock-formal}.

\section{Impossibility of Output-Detection Filters}

\paragraph{Preliminaries.}
A harmfulness oracle is a function $\mathcal{H} : \mathcal{T}^* \rightarrow [0,\infty)$.
For an output of a language model $z \in \mathcal{T}^*$ we interpret $\mathcal{H}(z)$ as the harmfulness of $z$, where $\mathcal{H}(z) = 0$ means that $z$ is benign and a high-value denotes high harmfulness.\footnote{It is important not to confuse the notation for harmfulness oracle $\mathcal{H}$ with the notation for entropy $H$.} 

The following is a formal version of Theorem~\ref{thm:detection-impossible-output}.

\begin{theorem}[Impossibility of output-detection filters]
\label{thm:detection-impossible-output-formal}
There exists $\e \in (0,1)$, polynomial $p$, such that for all security parameters $\secp \in \N$, 
every set of harmful prompts $S_H \subseteq \{0,1\}^{l_m}$, where $l_m = \Theta(p(\secp)), \Omega(p(\secp)) = \log(|S_H|) \leq l_m - \Omega(p(\secp))$, 
for every language model $M$ running in time $\poly(\secp)$ per token and satisfying $H(M) \geq \Omega(p(\secp))$, and every harmfulness oracle $\mathcal{H}$ the following conditions hold.

There exists an adversarial language model $M'$ with a running time $\poly(\secp)$ per token and a harmfulness oracle $\mathcal{H}'$ such that the following conditions hold.
\begin{enumerate}
    \item{\textbf{($M'$ outputs are harmful as judged by $\mathcal{H}'$)}.} For every malicious prompt $m \in S_H$, with all but negligible in $\secp$ probability over $z \gets M'(m)$,\footnote{Note that $M'(m)$ abuses the notation slightly. $M'$ is a malicious model that can, for every $m$, generate a harmful output with $m$ \say{hidden} in it. However, $m$ \emph{is not} an input/prompt to $M'$ in the usual sense.} it holds that
    \[
    \mathcal{H}'(z) = \Omega \Big(
    \mathbb{E}_{z \gets M(m)}[\mathcal{H}(z)] \Big).
    \]
    In words, for every prompt from $S_H$ the outputs of $M'$ are almost always as harmful (when judged by $\mathcal{H}'$) as if $M$ run on $m$ (when judged by $\mathcal{H}$).
    \item{\textbf{($\mathcal{H}'$ and $\mathcal{H}$ are equal on $M()$ almost always)}.} With all but negligible probability in $\secp$ over $z \gets M()$\footnote{We emphasize that $M()$ is \emph{not} a typo. $M()$ can be thought of as $G$ from Theorem~\ref{thm:detection-impossible-output-formal} that also doesn't receive any input.}
    \[
    \mathcal{H}'(z) = \mathcal{H}(z).
    \]
    In words, the harmfulness of outputs of $M$ is (almost always) equal as judged by $\mathcal{H}$ and $\mathcal{H}'$.
    \item{\textbf{($M'$ and $M$ are indistinguishable for filters running in $2^{\eps \cdot \secp}$)}.} Under the assumption that RSW is secure, for every output filter running in time $2^{\eps \cdot \secp}$, for every $m \in S_H$ the advantage in distinguishing outputs of $M()$ and $M'(m)$ is negligible in $\secp$.
    
    In words, no filter running in time $2^{\eps \cdot \secp}$ can reject\footnote{If it doesn't also reject the innocent prompts.} the harmful outputs from $M'$.
\end{enumerate}

\end{theorem}

\subsection{Construction}

As we explained in Section~\ref{sec:overview} the construction is a \say{lift} from the construction given in the proof of Theorem~\ref{thm:detection-impossible-timelock} (Theorem~\ref{thm:detection-impossible-timelock-formal} in the formal version).
More concretely we define
\begin{align*}
&M' := \mathsf{HarmIncite}(M,m,t = 2^\secp, \poly(\secp))  \\
&\mathcal{H}' := \mathsf{MaliciousLLM}(\mathcal{H},z,M,t = 2^\secp,\poly(\secp),S_H).
\end{align*}
We note that $\mathcal{H'}$ is well defined because the only calls to $\mathcal{H}$ in $\mathsf{MaliciousLLM}$ (Algorithm~\ref{alg:maliciousllmformal}) are query calls.
By direct, we mean that, for example, there is no assumption that $\mathcal{H}$ needs to be a language model.

\begin{remark}
We note that we could have set $t$ to a different (smaller) value and then the theorem would still hold with the security against filters running in time $t^\eps$.
This change would make sure that the \say{running time} of $\mathcal{H}'$ is smaller than $2^\secp$.
It can be beneficial if one tries to make sure that $\mathcal{H}'$ is efficient.
Thus our construction can be easily changed to have a trade-off between the overhead \say{running-time} (over $\mathcal{H}$) of $\mathcal{H}'$ and the security guarantee.
\end{remark}

\subsection{Proof}

The proof is a direct adaptation of the proof from Section~\ref{sec:mainformalproof}.
The main difference is in the setting of parameters for the time bounds.
In Theorem~\ref{thm:detection-impossible-timelock-formal} we needed $M'$ to be able to solve the time-lock puzzle (evaluate $r^{2^t}$) and so the indistinguishability could hold only against filters running in time $t^\eps$.
However, for the output filter result, the quality oracle could \say{run} in exponential time, and so the time-lock can be created with parameter $2^\secp$.\footnote{It cannot be \say{much} bigger because the running time of $M'$ is $\poly(\secp)$ and so $M'$ cannot create a harder puzzle (at least with the RSW construction).}

\section{Mitigation and Connections to Watermarking}

\subsection{Watermarks}

First, we define watermarking schemes, for which we adopt definitions from \cite{CGZ}.
The only difference is that we require a watermark to still be detected even if the output of the watermarking scheme was changed by a transformation from a class $E$.

\begin{definition}[Class of Transformations]
A \emph{class of transformations} $E$ is an equivalence relation over $\mathcal{T}^*$.
We interpret it as follows.
For $z \in \mathcal{T}^*$ it can be transformed into any element in the equivalence class of $z$.
We denote by $E(z)$ the equivalence class of $z$.
\end{definition}

\begin{definition}[Watermarking scheme]
A watermarking scheme for a model $M$ over $\mathcal{T}$ is a tuple of algorithms $\mathcal{W} = (\mathsf{Setup}, \mathsf{Wat}, \mathsf{Detect})$ where:
\begin{itemize}
  \item $\mathsf{Setup}(1^\lambda) \to \mathsf{sk}$ outputs a secret key, with respect to a security parameter $\lambda$.
  \item $\mathsf{Wat}_\mathsf{sk}(\mathsf{prompt})$ is a randomized algorithm that takes as input a prompt $\mathsf{prompt}$ and generates a response in $\mathcal{T}^*$.
  \item $\mathsf{Detect}_{\mathsf{sk}}(z) \to \{\mathsf{true}, \mathsf{false}\}$ is an algorithm that takes as input a sequence $z \in \mathcal{T}^\star$ and outputs true or false.
\end{itemize}
\end{definition}

\begin{definition}[Soundness]
A watermarking scheme $\mathcal{W}$ is \emph{sound} if for every security parameter $\lambda$ and token sequence $x \in \mathcal{T}^\star$ of length $|z| \le \mathsf{poly}(\lambda)$,
\[
\Pr_{\mathsf{sk} \leftarrow \mathsf{Setup}(1^\lambda)}[\mathsf{Detect}_{\mathsf{sk}}(z) = \mathsf{true}] \le \mathsf{negl}(\lambda).
\]
\end{definition}

\begin{definition}[$b(L)$-Completeness]
An algorithm $\mathcal{W}$ is a $b(L)$-\emph{complete} watermarking scheme robust against a class of transformations $E$ if for every security parameter $\lambda$ and prompt $\mathsf{prompt}$ of length $|\mathsf{prompt}| \le \mathsf{poly}(\lambda)$,
\begin{align*}
\Pr_{\mathsf{sk} \leftarrow \mathsf{Setup}(1^\lambda),\, z \leftarrow \mathsf{Wat}_{\mathsf{sk}}(\mathsf{prompt})}&[\exists_{z' \in E(z)} \ \mathsf{Detect}_{\mathsf{sk}}(z') = \mathsf{false} \text{ and } \\
&H_e(M, \mathsf{prompt},z) \ge b(|z|)] \le \mathsf{negl}(\lambda).
\end{align*}
\end{definition}

\begin{definition}[Undetectability]
A watermarking scheme $\mathcal{W} = (\mathsf{Setup}, \mathsf{Wat}, \mathsf{Detect})$ is \emph{undetectable} if for every security parameter $\lambda$ and all polynomial-time distinguishers $D$,
\[
\left| 
\Pr[D^{M, M}(1^\lambda) = 1]
-
\Pr_{\mathsf{sk} \leftarrow \mathsf{Setup}(1^\lambda)}[D^{M, \mathsf{Wat}_{\mathsf{sk}}}(1^\lambda) = 1]
\right| \le \mathsf{negl}(\lambda),
\]
where the notation $D^{O_1, O_2}$ means that $D$ is allowed to adaptively query both $O_1$ and $O_2$ with arbitrary prompts.
\end{definition}

\subsection{Our result}

\begin{definition}[Prompt-mitigation Filter]
A \emph{prompt-mitigation filter} for a class $E$ is an algorithm that takes as input $z \in \mathcal{T}^*$ and replies with en element of $E(z) \cup \{\bot\}$, which we interpret as it either produced an element in the equivalence class of $z$ or rejects.    
\end{definition}

The following is a formal version of Theorem~\ref{thm:mitigation-impossible}.

\begin{theorem}[Impossibility of prompt-mitigation filters]
\label{thm:watermarking-mitigation-impossible-formal}
If there exists a watermarking scheme robust against a class of transformations $E$ then
for all security parameters $\secp \in \N$, for every language model $M$, for every 
prompt generator $G$ satisfying $H(G) \geq \poly(\secp)$ and running in time $\poly(\secp)$ per token the following conditions hold.

There exists a language model $M'$ with a running time $\poly(\secp) + O(T_M)$ per token, where $T_M$ is the running time of $M$ per token, and an adversarial prompt generator $G'$ with a running time $\poly(\secp)$ per token such that if $G'$ and $M'$ share a secret key then the following conditions hold.
\begin{enumerate}
    \item{\textbf{($G'$ incites harmful behavior in $M'$ even in the presence of prompt-mitigation filters)}.} 
    For every prompt-mitigation filter $F$ running in polynomial-time with all but negligible in $\secp$ probability over $z \gets G'$ it holds that 
    \[
    F(z) = \bot \text{ or } M'(z) \text{ is harmful}.
    \]
    \item{\textbf{($M'$ and $M$ behave identically on $G$ almost always)}.} With all but negligible in $\secp$ probability over $z \gets G$ the distributions $M'(z)$ and $M(z)$ are equal.
    \item{\textbf{($G$ and $G'$ are indistinguishable for filters running in polynomial-time)}.} 
    For every algorithm running in polynomial-time the advantage in distinguishing outputs of $G$ and $G'$ is negligible in $\secp$.
\end{enumerate}
\end{theorem}

\begin{note}
Observe that conditions 1 and 3 from Theorem~\ref{thm:watermarking-mitigation-impossible-formal} imply that if a polynomial-time prompt-mitigation filter $F$ does not reject outputs of $G$ then the outputs of $G$ even when passed through the mitigation of $F$ will still incite harmful behavior in $M'$.   
\end{note}

\subsection{Construction}

Let $\mathcal{W} = (\mathsf{Setup}, \mathsf{Wat}, \mathsf{Detect})$ be a watermarking scheme robust against a class of transformations $E$.
Let $G$ be a prompt generator satisfying $H(G) \geq \poly(\secp)$.
Firstly, we let $G'$ and $M$ share a $\mathsf{sk}$ generated by $\mathsf{Setup}(1^\secp)$.
We define $G' := \mathsf{Wat}_\mathsf{sk}()$.\footnote{We run $\mathsf{Wat}$ on an empty prompt.}
Note that $\mathsf{Wat}$ implicitly depends on $G$.
For a language model $M$ we define $M'$ that on input $z$ works as follows:
\begin{enumerate}
    \item If $\mathsf{Detect}_\mathsf{sk}(z) = \mathsf{true}$ then return a harmful output.
    \item Return $M(z)$ otherwise.
\end{enumerate}

\subsection{Proof}

The proof requires a careful comparison of requirements of Theorem~\ref{thm:watermarking-mitigation-impossible-formal} and properties of watermarking schemes robust against classes of transformations.
The result follows from the following.
\begin{itemize}
    \item Condition 1 follows from completeness of $\mathcal{W}$.
    \item Condition 2 follows from soundness of $\mathcal{W}$.
    \item Condition 3 follows from undetectability of $\mathcal{W}$.
\end{itemize}

\end{document}